\documentclass{article}

\usepackage{microtype}
\usepackage{graphicx}
\usepackage{subfigure}
\usepackage{booktabs} 

\usepackage{hyperref}

\usepackage[accepted]{icml2024}

\usepackage{amsmath}
\usepackage{amssymb}
\usepackage{mathtools}
\usepackage{amsthm}

\usepackage{enumitem}
\setlist[itemize]{noitemsep, nolistsep, leftmargin=*}
\setlist[enumerate]{noitemsep, nolistsep, leftmargin=*}

\usepackage[textsize=tiny]{todonotes}

\newcount\Comments  
\newcommand{\kibitz}[2]{\ifnum\Comments=1\textcolor{#1}{#2}\fi}
\usepackage{color}
\definecolor{darkgreen}{rgb}{0,0.4,0}
\definecolor{purple}{rgb}{1,0,1}

\newcommand{\cready}[1]{{}}
\usepackage{FMP}

\newenvironment{talign}
 {\align}
 {\endalign}
\newenvironment{talign*}
 {\csname align*\endcsname}
 {\endalign}

\icmltitlerunning{tinyBenchmarks: evaluating LLMs with fewer examples}

\begin{document}

\twocolumn[
\icmltitle{tinyBenchmarks: evaluating LLMs with fewer examples}



\icmlsetsymbol{equal}{*}

\begin{icmlauthorlist}
\icmlauthor{Felipe Maia Polo}{mmm}
\icmlauthor{Lucas Weber}{ppp}
\icmlauthor{Leshem Choshen}{ibm,mit}
\icmlauthor{Yuekai Sun}{mmm}
\icmlauthor{Gongjun Xu}{mmm}
\icmlauthor{Mikhail Yurochkin}{ibm,mitibm}
\end{icmlauthorlist}

\icmlaffiliation{mmm}{Department of Statistics, University of Michigan, USA}
\icmlaffiliation{ppp}{Department of Translation and Language Sciences, University of Pompeu Fabra, Spain}
\icmlaffiliation{mit}{MIT}
\icmlaffiliation{ibm}{IBM Research}
\icmlaffiliation{mitibm}{MIT-IBM Watson AI Lab}

\icmlcorrespondingauthor{Felipe Maia Polo}{felipemaiapolo@gmail.com}
\icmlkeywords{IRT, Efficient, Benchmarking, LLM, Machine Learning}

\vskip 0.3in
]



\printAffiliationsAndNotice{} 

\begin{abstract}
The versatility of large language models (LLMs) led to the creation of diverse benchmarks that thoroughly test a variety of language models' abilities. These benchmarks consist of tens of thousands of examples making evaluation of LLMs very expensive. In this paper, we investigate strategies to reduce the number of evaluations needed to assess the performance of an LLM on several key benchmarks. For example, we show that to accurately estimate the performance of an LLM on MMLU, a popular multiple-choice QA benchmark consisting of 14K examples, it is sufficient to evaluate this LLM on 100 curated examples. We release evaluation tools and tiny versions of popular benchmarks: Open LLM Leaderboard, MMLU, HELM, and AlpacaEval 2.0.
Our empirical analysis demonstrates that these tools and tiny benchmarks are sufficient to reliably and efficiently reproduce the original evaluation results\footnote{To use our methods for efficient LLM evaluation, please check \url{https://github.com/felipemaiapolo/tinyBenchmarks}. This repository includes a Python package for model evaluation and tutorials. Additionally, we have uploaded tiny datasets on \href{https://huggingface.co/tinyBenchmarks}{\texttt{huggingface.co/tinyBenchmarks}} and developed a \href{https://github.com/felipemaiapolo/tinyBenchmarks/blob/main/tinyBenchmarks_MMLU_demo.ipynb}{Google Colab demo} in which you can easily use our tools to estimate LLM performances on MMLU. To reproduce the results in this paper, please check this \href{https://github.com/felipemaiapolo/efficbench}{GitHub repository}.}.
\end{abstract}

\section{Introduction}

Large Language Models (LLMs) have demonstrated remarkable abilities to solve a diverse range of tasks \citep{brown2020language}. Quantifying these abilities and comparing different LLMs became a challenge that led to the development of several key benchmarks, e.g., MMLU \citep{hendrycks2020measuring}, Open LLM Leaderboard \citep{open-llm-leaderboard}, HELM \citep{liang2022holistic}, and AlpacaEval \citep{alpaca_eval}. 

These benchmarks are comprised of hundreds or thousands of examples, making the evaluation of modern LLMs with billions of parameters computationally, environmentally, and financially very costly. For example, \citet{liang2022holistic} report that evaluating the performance of a single LLM on HELM costs over 4K GPU hours (or over \$10K for APIs). Benchmarks like AlpacaEval \citep{alpaca_eval} also require a commercial LLM as a judge to perform evaluation, further increasing the costs. Furthermore, evaluation of a single model is often performed many times to monitor checkpoints during pre-training \citep{biderman2023emergent,liu2023llm360} and to explore different prompting strategies or a wider range of hyperparameters \citep{weber2023mind,mizrahi2023state,sclar2023quantifying,voronov2024mind}.





\begin{figure}[h]
\centering
\includegraphics[width=.8\linewidth]{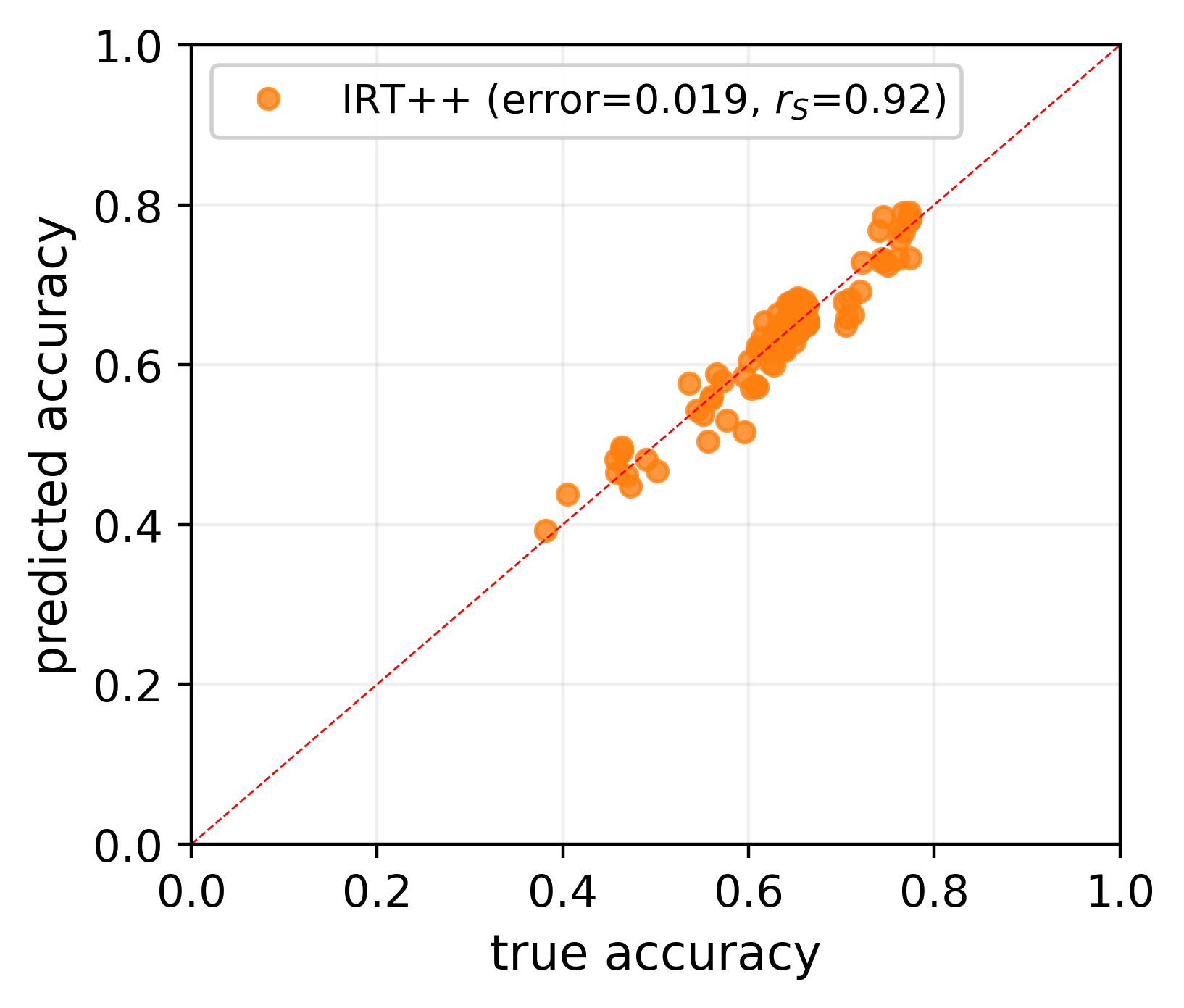}
\vspace{-0.5cm}
\caption{Estimating accuracy on MMLU (true accuracy) using 100 curated examples (predicted accuracy). IRT++, our best-performing evaluation strategy, predicts the accuracy of recent LLMs released between December 30th and January 18th within 1.9\% of their true accuracy on all of MMLU (14K examples).}
\label{fig:mmlu-intro}
\end{figure}

Our work reassesses the need to evaluate LLMs on such large benchmark datasets. In Figure \ref{fig:mmlu-intro} we demonstrate the efficacy of our best evaluation strategy on MMLU, where we compare accuracy estimates obtained from evaluating LLMs on a curated subset of 100 examples (less than 1\% of the examples) to accuracy on all of MMLU, achieving average estimation error under 2\%.

We consider a range of evaluation strategies (\S\ref{sec:strategies}):
\begin{enumerate}
    \item Stratified random sampling as proposed by \citet{perlitz2023efficient} for HELM. This approach is the simplest to use but can result in a large estimation error.
    \item Clustering examples based on LLMs that have already been evaluated. The key idea is to find examples where (in)correct prediction of an LLM implies that it will also be (in)correct on a subset of other examples. This method performs well in some settings but can be unreliable when such correctness patterns are spurious, e.g., when predicting the accuracy of an LLM specialized to a domain. This strategy is inspired by the Anchor Points method \citep{vivek2023anchor} which clusters models' confidence in the correct class for faster evaluation on classification tasks.
    \item New strategies built using Item Response Theory (IRT) \citep{lord1968statistical} for evaluating individuals through standardized tests. Applying IRT to LLMs viewed as testees and benchmarks as tests, we learn representations of examples encoding latent abilities required to perform well on these examples. Clustering these representations allows us to find a more robust evaluation set. Furthermore, using the IRT model, we develop tools for improving benchmark accuracy estimates obtained with an arbitrary set of examples.
\end{enumerate}


We present an extensive evaluation of these strategies on four popular benchmarks (\S\ref{sec:experiments}): Open LLM Leaderboard \citep{open-llm-leaderboard}, MMLU \citep{hendrycks2020measuring}, HELM \citep{liang2022holistic}, and AlpacaEval 2.0 \citep{alpaca_eval}. Our goal is to assess the effectiveness of estimating the performance of LLMs on these benchmarks using a limited number of examples for evaluation. Overall, we conclude that 100 curated examples per scenario are enough to reliably estimate the performance of various LLMs, within about 2\% error on average. Based on our findings we release tiny (100 examples per scenario) versions of every considered benchmark and IRT-based tools for further improving the performance estimation.

\subsection{Related work}

\paragraph{Efficient benchmarking of LLMs}
Multi-dataset benchmarks were introduced to the field of NLP  with the advent of pre-trained models \citep[e.g.][]{wang2018glue}, and constantly evolved in lockstep with language model capabilities \citep{srivastava2022beyond}. The ever-increasing size of models and datasets consequently led to high evaluation costs, triggering changes in reported evaluation to accommodate the costs \citep{Biderman2023PythiaAS}.
\citet{ye2023predictable} considered reducing the number of \emph{tasks} in Big-bench \citep{srivastava2022beyond}.
\citet{perlitz2023efficient} found that evaluation on HELM \citep{liang2022holistic} relies on diversity across datasets, but the number of examples currently used is excessive. We adopt their stratified sampling approach as one of the efficient evaluation strategies.
\citet{vivek2023anchor} proposed clustering evaluation examples based on models' confidence in the correct class for faster evaluation on classification tasks. One of the approaches we consider is based on an adaptation of their method to popular LLM benchmarks with more diverse tasks.

\paragraph{Item response theory (IRT)}


IRT
\citep{cai2016item, van2018handbook, brzezinska2020item,lord1968statistical} is a well-established set of statistical models used in psychometrics to measure the latent abilities of individuals through standardized testing \citep{an2014item, kingston1982feasibility, petersen1982using}, \eg, in GRE, SAT, \etc. Even though IRT methods have been traditionally used in psychometrics, they are becoming increasingly popular among researchers in the fields of artificial intelligence and natural language processing (NLP). For instance, \citet{lalor2016building} propose using IRT's latent variables to measure language model abilities, \citet{vania2021comparing} employs IRT models in the context of language models benchmarking to study saturation (un-discriminability) of commonly used benchmarks, and \citet{rodriguez-etal-2021-evaluation} study several applications of IRT in the context of language models, suggesting that IRT models can be reliably used to: predict responses of LLMs in unseen items, categorize items (\eg, according to their difficulty/discriminability), and rank models. More recently, \citet{zhuang2023efficiently} used IRT for adaptive testing, making testing more efficient. However, the authors do not propose a performance estimator for LLMs but only rank models based on their ability parameters. To the best of our knowledge, IRT has not been used for performance estimation in the context of efficient benchmarking of LLMs. We explore this new path.

\paragraph{Active testing}
Another line of related work is related to active learning \cite{ein-dor-etal-2020-active} and especially active testing. In such works, evaluation examples are chosen dynamically using various criteria \citep{Ji_Logan_Smyth_Steyvers_2021,kossen2021active,zhuang2023efficiently} to minimize annotation costs.
Those methods are somewhat similar to the adaptive IRT which we discuss in \S\ref{sec:concl}.

\section{Problem statement}\label{sec:problem}
In this section, we describe in detail the setup we work on and what are our objectives. Consider that a benchmark is composed of scenarios and possibly sub-scenarios. For example, MMLU and HellaSwag are examples of scenarios\footnote{We consider  MMLU and AlpacaEval as a single scenario each.} of both the Open LLM Leaderboard and HELM, while MMLU has different sub-scenarios like ``marketing", ``elementary mathematical", and so on. Furthermore, each scenario (or sub-scenario) is composed of examples (analogous to ``items" in the IRT literature) that are small tests to be solved by the LLMs--these examples range from multiple-choice questions to text summarization tasks. Our final objective is to estimate the performance of LLMs in the full benchmark, which is given by the average of the performances in individual scenarios (Open LLM Leaderboard, MMLU, AlpacaEval 2.0) or mean-win-rate (HELM). We achieve this objective by first estimating the performance of LLMs in individual scenarios and then aggregating scores. When scenarios have sub-scenarios, it is usually the case that the scenario performance is given by a simple average of sub-scenarios performances. The main concern is that each scenario/sub-scenario is composed of hundreds or thousands of examples, making model evaluation costly. 

In this work, for a fixed benchmark, we denote the set of examples of each scenario $j$ as $\cI_j$, implying that the totality of examples in the benchmark is given by $\cI=\cup_{j}\cI_j$. When an LLM $l$ interacts with an example $i \in \cI_j$, the system behind the benchmarks generates a score that we call ``correctness" and denote as $Y_{il}$. In all the benchmarks we consider in this work, the correctness is either binary, \ie, $Y_{il}\in \{0,1\}$ (incorrect/correct), or bounded, \ie, $Y_{il}\in [0,1]$, denoting a degree of correctness. The second case is applied in situations in which, for instance, there might not be just one correct answer for example $i$. To simplify the exposition in the text, we assume that the score for LLM $l$ in scenario $j$ is just the simple average of the correctness of all items in that scenario, that is, $ \frac{1}{|\cI_j|}\sum_{i\in \cI_j} Y_{il}$. That is not true when different sub-scenarios have different numbers of examples; in that case, one would just have to use a weighted average instead, to make sure every sub-scenario is equally important (in the experiments, we consider this case).

Our objective is to choose a small fraction of examples $\widehat{\cI}_j\subset \cI_j$ such that we can estimate score of a new LLM $l$, \ie, $ \frac{1}{|\cI_j|}\sum_{i\in \cI_j} Y_{il}$, using its correctness evaluated \emph{only} on the examples in $\widehat{\cI}_j\subset \cI_j$, \ie, $\{Y_{il}\}_{i \in \widehat{\cI}_j}$. To intelligently choose $\widehat{\cI}_j$ we assume access to correctness evaluations for a set of LLMs that have been previously evaluated on the entirety of the benchmark. Such correctness data is freely available for many popular benchmarks. In the next section, we describe strategies on how $\widehat{\cI}_j$ can be chosen and how the LLMs performance on the full benchmark can be estimated.

\section{Selecting evaluation examples}\label{sec:strategies}

In this section, we describe strategies on how to select examples from a fixed scenario $j$, \ie, $\cI_j$, obtaining $\widehat{\cI}_j\subset \cI_j$ described in Section \ref{sec:problem}. Ideally, the set of selected examples should be representative of the whole set of items in scenario $j$, that is,
\begin{talign}\label{eq:estimate}
    \sum_{i\in \widehat{\cI}_j}w_i Y_{il} \approx \frac{1}{|\cI_j|}\sum_{i\in \cI_j} Y_{il},
\end{talign}
for nonnegative weights $\{w_i\}_{i\in \widehat{\cI}_j}$ such that $\sum_{i\in \widehat{\cI}_j}w_i=1$. In the next paragraphs, we describe two possible ways of obtaining $\widehat{\cI}_j$ and  $\{w_i\}_{i\in \widehat{\cI}_j}$.

\subsection{Stratified random sampling} 

In some settings \citep[e.g., classifiers][]{katariya2012active}, it is useful to perform stratified random sampling -- subsample examples ensuring the representation of certain groups of data.
Using subscenarios as the strata for stratified random sampling was proposed by \citet{perlitz2023efficient} when sub-sampling examples from HELM scenarios. The authors showed that this is an effective way of sampling examples without too much loss on the ability to rank LLMs by performance.  Examples should be randomly selected from sub-scenarios (with uniform probability) in a way such that the difference in number of examples sampled for two distinct subscenarios is minimal ($\le1$). The rationale behind this method is that, for an effective evaluation, sub-scenarios should be equally represented. The weights are $w_i=1/|\widehat{\cI}_j|$ for all $i\in \widehat{\cI}_j$.

\subsection{Clustering} 

Assessing the performance of LLM's on a randomly sampled subset of examples suffers from extra uncertainty in the sampling process, especially when the number of sampled examples is small. Instead, we consider selecting a subset of representative examples using clustering. \citet{vivek2023anchor} proposed to cluster examples based on the confidence of models in the correct class corresponding to these examples. Representative examples, from these clusters, which they call ``anchor points'', can then be used to evaluate models on classification tasks more efficiently. We adapt their clustering approach to a more general setting, allowing us to extract such anchor points for MMLU, AlpacaEval 2.0, and all scenarios of the Open LLM Leaderboard and HELM.


First, we propose to group examples by model correctness, expecting some examples would represent the rest. Ideally, if example $i$ is an anchor point, then there will be a big set of examples on which models are correct if and only if they get example $i$ correct. The same idea applies when correctness is given by a number in $[0,1]$. Assume that we want to select $K$ anchor points and have access to the training set $\cD_{tr}=\{Y_{l}\}_{l \in \cL_{tr}}$, where $Y_{l}$ is a vector in which each entry is given by the correctness score $Y_{il}$ for all examples $i\in \cI_j$. We represent each example $i\in \cI_j$ by the embedding $E_i\in \reals^{|\cL_{tr}|}$ which is a vector with entries given by $Y_{il}$ for $l\in \cL_{tr}$, and then run $K$-Means \citep{hastie2009elements} with the number of clusters being equal $K$. After the $K$ centroids are obtained, we find the closest example to each centroid, and each of those points will compose  $\widehat{\cI}_j$. For a new LLM $l\not\in \cL_{tr}$ to be evaluated, we can obtain an estimate for its performance using the estimate in equation \ref{eq:estimate} by setting $w_i$ as the fraction of points in $\cI_j$ assigned to cluster/anchor point $i$. This method is compelling and simple in detecting anchor points. Still, it can suffer from distribution shifts since correctness patterns can vary, \eg, in time, and from the curse of dimensionality when $|\cL_{tr}|$ is big. Our second approach is intended to be more robust to those problems.

The second approach we propose is using item response theory (IRT) representation of examples, detailed in Section \ref{sec:irt}, as our embeddings $E_i$. The IRT model creates a meaningful representation for each example $i$ based on their difficulty and the abilities required to respond to those examples correctly. This approach immediately solves the dimensionality problem, since $E_i$ is relatively low-dimensional\footnote{In our experiments, the dimension of $E_i$ is $\leq16$.}, and potentially alleviates the distribution shift problem if the IRT model reasonably describes the reality and the example representations are stable. As IRT should represent which examples have similar difficulty and require similar abilities, the anchors represent exactly what we looked for. The weight $w_i$ is given by the fraction of examples in $\cI_j$ assigned to cluster/anchor point $i$.

\section{Better performance estimation with IRT}\label{sec:irt}

In this section, we propose ways of enhancing performance estimates by using IRT models. We start by discussing the case where $Y_{il}\in\{0,1\}$, that is, the $l$ responds to the example $i\in \cI$ correctly or not. We later also discuss the case where $Y_{il}\in[0,1]$.

\subsection{The IRT model}
The two-parameter multidimensional IRT model assumes that the probability of the LLM $j$ getting example $i$ correctly is given by
\begin{talign}\label{eq:irt_model}
p_{il}&\triangleq\Pr(Y_{il} = 1\mid \theta_l, \alpha_i, \beta_i) =\frac{1}{1 + \exp(- \alpha_i^\top\theta_l + \beta_i)},
\end{talign}
where $\theta_l\in\reals^d$ denotes the unobserved abilities of LLM $l$, while $\alpha_i\in\reals^d$ dictates which dimensions of  $\theta_l$ are required from model $l$ to respond to example $i$ correctly. In this formulation, $\beta_i\in\reals$ can be viewed as a bias term that regulates the probability of correctness when $\theta_l=0$. We use IRT parameter estimates as example representations referred to in Section \ref{sec:strategies}. Specifically, we take $E_i=(\widehat{\alpha}_i, \widehat{\beta}_i)$, where $\widehat{\alpha}_i$ and $\widehat{\beta}_i$ are point estimates for the parameters of example $i$. In the next sections, we introduce two estimators for the performance of an LLM, propose a simple solution for the case $Y_{il}\not\in\{0,1\}$, and describe model fitting.

\subsection{IRT-based LLM performance estimation}

\paragraph{The performance-IRT (p-IRT) estimator.} Assume that we are interested in estimating the performance of a model $l\not \in \cL_{tr}$ on scenario $j$ and that point estimates of example parameters, $(\widehat{\alpha}_i,\widehat{\beta}_i)$, have been computed, using a training set, for all examples in all scenarios, including examples $i \in \cI_j$. Formally, we are interested in approximating 
\begin{talign}
    Z_{jl}\triangleq\frac{1}{|\cI_j|}\sum_{i\in\cI_j}Y_{il}
\end{talign}
Now, assume that we have run model $l$ on a subset of examples from scenario $j$, obtaining responses $\{Y_{i_0 l}, \cdots, Y_{i_k l}\}$ for the examples $\widehat{\cI}_j = \{i_0,\cdots, i_k\}$. Let $\widehat{\theta}_l$ denote the estimate for $\theta_l$ after observing $\widehat{\cI}_j$ and possibly a bigger set of examples coming from different scenarios. To obtain that estimate, we maximize the log-likelihood of the freshly observed data with respect to $\theta_l$, fixing examples' parameters. This procedure is equivalent to fitting a logistic regression model, which is an instance of the well-studied $M$-estimation procedure. 


Because $Z_{jl}$ is a random variable, we approximate it by estimating the conditional expectation 
\begin{talign*}
    &\Ex[Z_{jl} \mid Y_{i_0 l}, \cdots, Y_{i_k l}]=\\
    &=\frac{1}{|{\cI}_j|}\sum_{i \in {\cI}_j}\Ex[Y_{il}\mid Y_{i_0 l}, \cdots, Y_{i_k l}]\\
    &=\frac{1}{|{\cI}_j|}\left(\sum_{i\in \widehat{\cI}_j} Y_{i l}+\sum_{i \in {\cI}_j\setminus\widehat{\cI}_j  }p_{il}\right)\\
    &=\frac{\hat{\lambda}}{|\widehat{\cI}_j|}\sum_{i\in \widehat{\cI}_j} Y_{i l}+\frac{1-\hat{\lambda}}{|{\cI}_j\setminus\widehat{\cI}_j|}\sum_{i \in {\cI}_j\setminus\widehat{\cI}_j  }p_{il}
\end{talign*}
which is the best approximation for $Z_{jl}$ in the mean-squared-error sense. Here, $\hat{\lambda} = |\widehat{\cI}_j|/|{\cI}_j| \in [0,1]$ is a weight that gives more or less importance to the observed set $\widehat{\cI}_j$ in the performance computation depending on how big that set is. The probability $p_{il}=\Pr(Y_{il} = 1\mid\theta_l, \alpha_i, \beta_i)$ is given by the IRT model in Equation \ref{eq:irt_model}. The estimator for the conditional expectation is then given by
\begin{talign}\label{eq:pirt}
    \hat{Z}^{\text{p-IRT}}_{jl} &\triangleq \widehat{\Ex}[Z_{jl} \mid Y_{i_0 l}, \cdots, Y_{i_k l}]\\
    &=\frac{\hat{\lambda}}{|\widehat{\cI}_j|}\sum_{i\in \widehat{\cI}_j} Y_{i l}+\frac{1-\hat{\lambda}}{|{\cI}_j\setminus\widehat{\cI}_j|}\sum_{i \in {\cI}_j\setminus\widehat{\cI}_j  } \hat{p}_{il}\nonumber
\end{talign}
where $\hat{p}_{il}\triangleq \Pr(Y_{il} = 1\mid \widehat{\theta}_l, \widehat{\alpha}_i, \widehat{\beta}_i)$. We call the estimator in \ref{eq:pirt} by Performance-IRT (p-IRT) estimator.

The idea behind p-IRT is that we can estimate the performance of a model on unseen data making use of the IRT model. This is especially useful if we can fit $\widehat{\theta}_l$ using data from many scenarios: even though we observe just a few samples per scenario, p-IRT will leverage the whole available data, permitting better estimates for the performance of the LLM for all scenarios. Conditional on the training set, the estimator p-IRT has low variance when $\widehat{\theta}_l$ is obtained from a large dataset and a small bias if the IRT model is reasonably specified. Given that $\widehat{\theta}_l$ is potentially estimated using a large sample, it is worth understanding what that implies about our estimates $\hat{Z}^{\text{p-IRT}}_{jl}$'s in the asymptotic regime. To facilitate our analysis, assume for a moment that the true values of $(\alpha_i, \beta_i)$'s for all $i\in\cI$ are known. As previously commented, estimating $\theta_l$ is equivalent to fitting a logistic regression and, under mild conditions, we should have $\hat{\theta}_l\to \theta_l$ in probability as $|\widehat{\cI}|\to\infty$ \citep{fahrmeir1985consistency}. We depart from this condition and show that $|\widehat{\Ex}[Z_{jl} \mid Y_{i_0 l}, \cdots, Y_{i_k l}]-\Ex[Z_{jl} \mid Y_{i_0 l}, \cdots, Y_{i_k l}]|\to 0$ in probability as $|\widehat{\cI}|\to\infty$; that is, p-IRT converges in probability to the best approximation of $Z_{jl}$, $\Ex[Z_{jl} \mid Y_{i_0 l}, \cdots, Y_{i_k l}]$.

\begin{proposition}\label{prop:pirt-consistency}
Assuming that  (i) $\hat{\theta}_l\to \theta_l$ in probability as $|\widehat{\cI}|\to\infty$ and that (ii) the true values of $(\alpha_i, \beta_i)$'s for all $i\in\cI$ are known and $\sup_{i\in\cI}\norm{\alpha_i}_2\leq c$ for a universal constant $c$, we have that
\[\textstyle
|\widehat{\Ex}[Z_{jl} \mid Y_{i_0 l}, \cdots, Y_{i_k l}]-\Ex[Z_{jl} \mid Y_{i_0 l}, \cdots, Y_{i_k l}]|\to 0
\]
in probability as $|\widehat{\cI}|\to\infty$.
\end{proposition}

We note two limitations of p-IRT that can hinder its effectiveness in practice. First, it does not promptly allow sample weighting, limiting its use of anchor points; second, if the predicted probabilities $\hat{p}_{il}$'s are inaccurate, \eg, because of model misspecification, then the performance of p-IRT will deteriorate.  

\paragraph{The generalized p-IRT (gp-IRT) estimator.} Our final estimator builds upon p-IRT to overcome its limitations. Assume that the estimators in equations \ref{eq:estimate} and \ref{eq:pirt} are obtained as a first step after the collection of examples in $\widehat{\cI}_j$. The idea is to compute a third estimator $\hat{Z}^{\text{gp-IRT}}_{jl}$ given by a convex combination of the first two
\begin{talign}\label{eq:gpirt}
    &\hat{Z}^{\text{gp-IRT}}_{jl} \triangleq \lambda\sum_{i\in \widehat{\cI}_j}w_iY_{il} + (1-\lambda)\hat{Z}^{\text{p-IRT}}_{jl}
\end{talign}
where $\lambda$ is a number in $[0,1]$ that is chosen to optimize the performance of that estimator. To choose $\lambda$, we first note that using random sampling (or anchor points) implies low bias but potentially high variance (when $\widehat{\cI}_j$ is small) for $\sum_{i\in \widehat{\cI}_j}w_iY_{il}$. As $\widehat{\cI}_j$ grows, its variance decreases. On the other hand, conditional on the training set, the variance of $\hat{Z}^{\text{p-IRT}}_{jl}$ is small, especially when $\widehat{\theta}_l$ is fitted with data from many scenarios, but its bias can be high when the IRT model is misspecified and does not vanish with the growing sample size. Thus, good choice of $\lambda$ increases with $\widehat{\cI}_j$.

We choose $\lambda$ based on a heuristic derived from \citet{song1988minimal}'s Corollary 2. It tells us that the optimal linear combination of any two estimators $\hat{T}_1$ and $\hat{T}_2$ (when the sum of the weights is one) depends on the biases, variances, and covariance of the two estimators. If the first estimator is unbiased and the variance of the second is zero, we can show that the optimal estimator is $\lambda\hat{T}_1 + (1-\lambda)\hat{T}_2$, where $\lambda = b^2_2/(b^2_2+v_1)$, $b_2$ denotes $\hat{T}_2$'s bias, and $v_1$ denotes $\hat{T}_1$'s variance. To apply this result, we assume that the main factors that might prevent gp-IRT from being a good estimator are the variance of the first estimator and the bias of the second one. Then we approximate the first estimator's bias and the second estimator's variance by zero. When our first estimator is obtained by random sampling we take
\[
\lambda = \frac{\hat{b}^2}{\hat{\sigma}^2/|\widehat{\cI}_j|+\hat{b}^2}
\]
for two constants $\hat{\sigma}^2$ and $\hat{b}^2$. The first constant, $\hat{\sigma}^2$, is obtained by computing the average sample variance of $Y_{il}$, $i\in\cI_j$, across LLMs in the training set. The second constant, $\hat{b}^2$, is obtained by approximating the IRT bias. We (i) split the training set into two subsets of LLMs; (ii) fit an IRT model in the first part using data from all scenarios; (iii) fit the ability parameter for all the LLMs in the second part using half of the examples of all scenarios; (iv) use that IRT model to predict the correctness (using predicted probabilities) of the unseen examples of scenario $j$ for the models in the second split; (v) average predictions and actual correctness within models, obtaining predicted/actual scenarios scores; (vi) compute their absolute differences, obtaining individual error estimates for models; (vii) average between models, obtaining a final bias estimate, and then square the final number. To give some intuition on how $\lambda$ is assigned, Figure \ref{fig:lambda} depicts $\lambda$ as a function of $\hat{b}$ and $|\widehat{\cI}_j|$ when $\hat{\sigma}^2=.01$. From that figure, we see that if the IRT model bias is small, more weight will be given to p-IRT. The curves are steeper when $|\widehat{\cI}_j|$ is small because the variance of the first estimator decreases faster when $|\widehat{\cI}_j|$ is small. 
\begin{figure}[h]
\centering
\includegraphics[width=.7\linewidth]{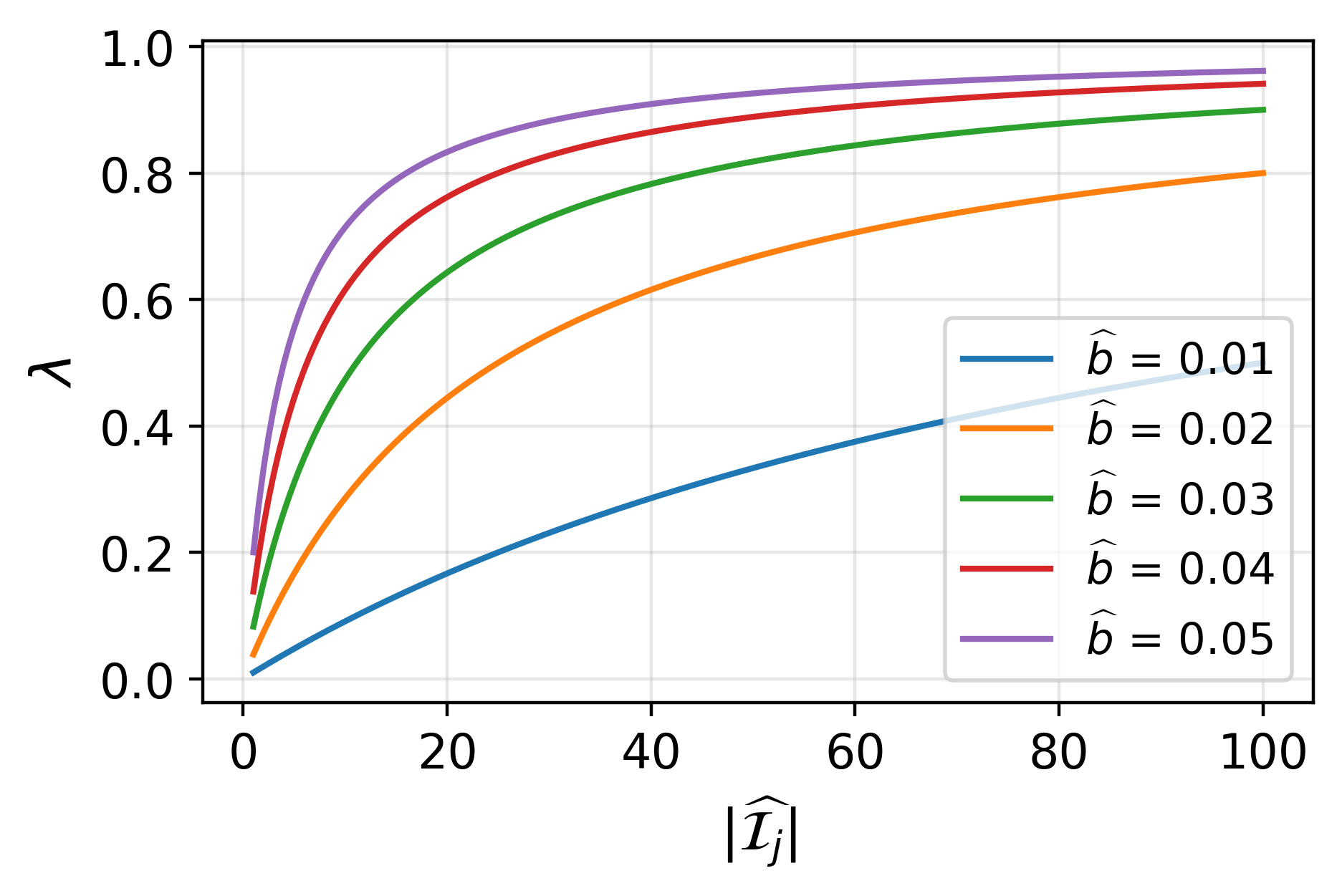}
\vspace{-0.5cm}
\caption{Understanding the effect of IRT bias and sample size $|\widehat{\cI}_j|$ in the gp-IRT construction: both quantities are positively related to the weight we give to the raw data in performance estimation.}
\label{fig:lambda}
\end{figure}
When the first estimator is obtained by a method that implies an estimator with smaller variance, \eg, anchor points, we apply the same formula but divide $\hat{\sigma}^2$ by a constant $>1$. By default, we divide $\hat{\sigma}^2$ by $4$ which is equivalent to halving the standard deviation of the first estimator.
\subsection{Using IRT when $Y_{il}$ is not binary}\label{sec:thresh}

There are situations in which $Y_{il}\notin \{0,1\}$ but $Y_{il}\in[0,1]$. For example, in AlpacaEval 2.0, the response variable is bounded and can be translated to the interval $[0,1]$. Also, some scenarios of HELM and the Open LLM Leaderboard have scores in $[0,1]$.
We propose a simple and effective fix. The idea behind our method is to binarize $Y_{il}$ by defining a second variable $\Tilde{Y}_{il} = \ones[Y_{il}\geq c]$, for a scenario-dependent constant $c$. More concretely, for each scenario $j$, we choose $c$ such that 
\begin{talign*}
    \sum_{i\in\cI_j,l\in\cL_{tr}} Y_{i l} \approx \sum_{i\in\cI_j,l\in\cL_{tr}}\ones[Y_{il}\geq c].
\end{talign*}
In that way,  approximating the average of $\Tilde{Y}_{il}$ and $Y_{i l}$ should be more or less equivalent. Given that $\Tilde{Y}_{il}\in\{0,1\}$, we can use the standard IRT tools to model it.

\subsection{Fitting the IRT model}\label{sec:fitting}

For the estimation procedure, we resort to variational inference. In particular, we assume that $\theta_l \sim N(\mu_{\theta}\ones_d, 1/u_{\theta}I_d)$, $\alpha_{i} \sim N(\mu_{\alpha}\ones_d, 1/u_{\alpha}I_d)$, and $\beta_{i} \sim N(\mu_\beta, 1/u_\beta)$. To take advantage of software for fitting hierarchical Bayesian models  \citep{lalor2023py}, we introduce (hyper)priors for the prior parameters $\mu_{\theta} \sim N(0, 10)$, $u_{\theta} \sim \Gamma(1, 1)$, $\mu_{\alpha} \sim N(0, 10)$, $u_{\alpha} \sim \Gamma(1, 1)$, $\mu_\beta \sim N(0, 10)$, and $u_\beta \sim \Gamma(1, 1)$. Finally, to obtain point estimates for the model and example-specific parameters $\theta_l$, $\alpha_{i}$, and $\beta_{i}$, we use the means of their variational distributions. To select the dimension of the IRT model during the fitting procedure, we run a simple validation strategy in the training set and choose the dimension that maximizes the prediction power of the IRT model in the validation split--we consider the dimensions in $\{2,5,10,15\}$.
\begin{figure*}[ht!]
\centering
\includegraphics[width=\textwidth]{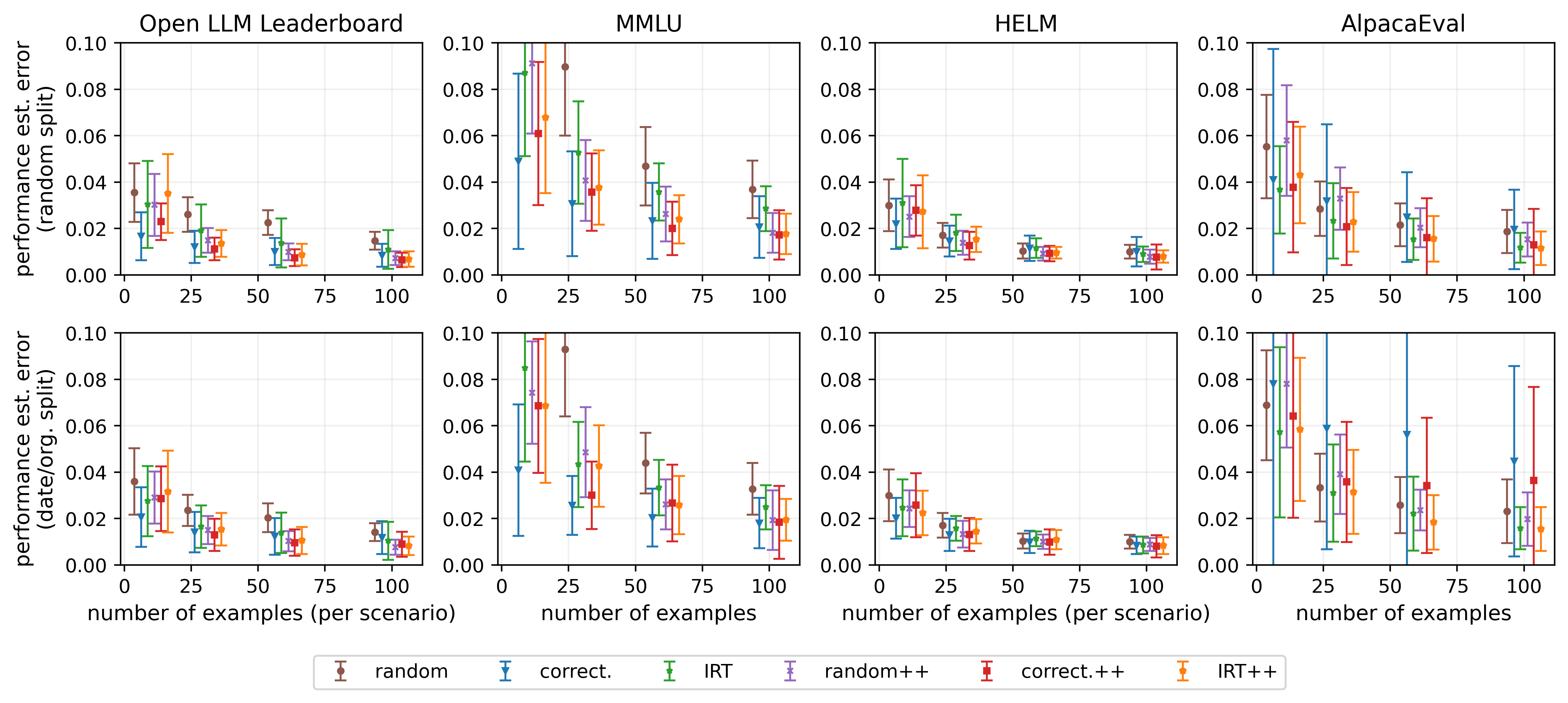}
\vspace{-0.75cm}
\caption{Performance estimation error per benchmark (columns) tested on random (top row) and recent (bottom row) LLMs for increasing number of evaluation examples. 100 examples per scenario is sufficient to achieve $\approx$2\% average performance estimation error across benchmarks and evaluated LLMs. This corresponds to 600 out of 29K examples for Open LLM Leaderboard, 100 out of 14K examples for MMLU, 1000 out of 10K examples for HELM, and 100 out of 800 examples for AlpacaEval 2.0.}
\label{fig:leaderboard_performance_acc}
\end{figure*}

\section{Assessing evaluation strategies}
\label{sec:experiments}
We assess the ability of the considered evaluation strategies to estimate the performance of LLMs on four popular benchmarks. For a given LLM and a benchmark, each evaluation strategy estimates the performance using evaluation results of this LLM on a given number of examples. We then compare this estimate to the true value, \ie, the performance of this LLM on the complete benchmark.

\paragraph{Evaluation pipeline}
For each benchmark, we first collect publicly available correctness data ($Y_{il}$'s) for a set of LLMs $\cL$ that have been previously evaluated on this benchmark. Recall that the benchmark is a set of examples $\cI$ consisting of $J$ disjoint scenarios examples $\cI_j$ such that $\cI=\cup_{j\in[J]}\cI_j$. We use correctness data corresponding to a subset of LLMs $\cL_{tr}$, \ie, $\cD_{tr}=\{Y_{il}\}_{l\in \cL_{tr}, i\in\cI}$ to (i) find anchor points $\hat{\cI}_j$ for each one of the scenarios $j\in [J]$ as described in Section \ref{sec:strategies} and (ii) to obtain estimates for the IRT parameters $\{(\alpha_i,\beta_i)\}_{i\in\cI}$ as described in Section \ref{sec:irt}. We call this ``train'' set of models as their correctness data is used to identify anchor points and fit the parameters associated with our evaluation strategies. The remaining set of ``test'' models $\cL_{te}$ is used to quantify the error of our evaluation strategies in practice. For each LLM in the test set, $l\in \cL_{te}$, we observe its correctness on the anchor points, \ie, $\{Y_{il}\}_{i\in\hat{\cI}_j}$, and use it to obtain benchmark performance estimates as described in Sections \ref{sec:strategies} and \ref{sec:irt}. The estimate is then compared to the ground truth, \ie, performance of this LLM on the entirety of the benchmark.

We consider two train-test model split scenarios: (i) random split and (ii) by date, \ie, using the most recent models for testing. The latter split better represents practical use cases, while also being more challenging as it is likely to result in a distribution shift between the train and test models due to improving model capabilities over time that might affect the effectiveness of anchor points and the IRT model. 

\paragraph{Benchmarks and models} 
We describe the size and composition of the four benchmarks, as well as the corresponding LLMs (see Appendix \ref{sec:benchmark} for additional details):
\begin{itemize}
    \item HuggingFace's Open LLM Leaderboard \citep{open-llm-leaderboard} consists of 6 scenarios, approx. 29K examples in total. Performance on each of the scenarios is measured with accuracy and the overall benchmark performance is equal to the average of scenario accuracies.
    We collect evaluation results for 395 LLMs from the Leaderboard's website and use 75\% for training and 25\% for testing (split either randomly or by date as described above).
    

    \item MMLU \citep{hendrycks2020measuring} is a multiple choice QA scenario consisting of 57 subjects (subscenarios) comprising approx. 14K examples. Performance on MMLU is measured by averaging the accuracies on each of the categories. MMLU is one of the 6 scenarios of the Open LLM Leaderboard and we consider the same set of 395 LLMs and train-test splits. The reason to consider it separately is its immense popularity when comparing LLMs \citep{touvron2023llama,achiam2023gpt,team2023gemini} and inclusion into several other benchmarks.

    \item For HELM \citep{liang2022holistic}, we use HELM Lite v1.0.0, which has the 10 core scenarios (total of approx. 10K evaluation examples) and 30 models that have their performances registered for all scenarios. Performance metrics for each scenario vary and can be non-binary (\eg, F1 score), and the overall performance on the benchmark is measured with mean win rate across scenarios.
    For this benchmark, the dates models were added are not available. Instead, we split models based on the organizations that trained them to create more challenging train-test splits, \eg, all OpenAI models are either in train or in test. For the random train-test split we use 11-fold cross-validation. That is, we partition the set of all LLMs into $k=11$ parts and, for each one of these parts, we use one of them to test and $k-1$ parts for training. Then, we average the results over the choice of the testing part.


    \item AlpacaEval 2.0 \citep{alpaca_eval} consists of 100 LLMs evaluated on 805 examples. Although it is a fairly small benchmark, evaluation is expensive as it requires GPT-4 as a judge. For each input, GPT-4 compares the responses of a candidate LLM and a baseline LLM (currently also GPT-4) and declares a winner. The average win rate\footnote{AlpacaEval 2.0 considered in the experiments uses continuous preferences instead of binary.} is used to measure the overall performance. When splitting the data by date, we pick $25\%$ most recent models for testing and the rest for training. For the random split, we employ 4-fold cross-validation analogous to HELM.
    
\end{itemize}



\paragraph{Evaluation strategies}
We consider 3 strategies presented in \S\ref{sec:strategies} for selecting a subset of examples for efficient evaluation: ``random'' for stratified random sampling, ``correctness'' for clustering correctness of models in the train set, and ``IRT'' for clustering the example representations obtained from the IRT model fit on the train set. For each strategy, we evaluate the vanilla variation, \ie, simply using the performance of a test LLM on the (weighted) set of selected examples to estimate its performance on the full benchmark, and ``++'' variation that adjusts this estimate using the IRT model as described in equation \eqref{eq:gpirt}. In total, we assess six evaluation strategies. Results are averaged over 5 restarts.

\begin{figure*}[ht!]
\centering
\includegraphics[width=\linewidth]{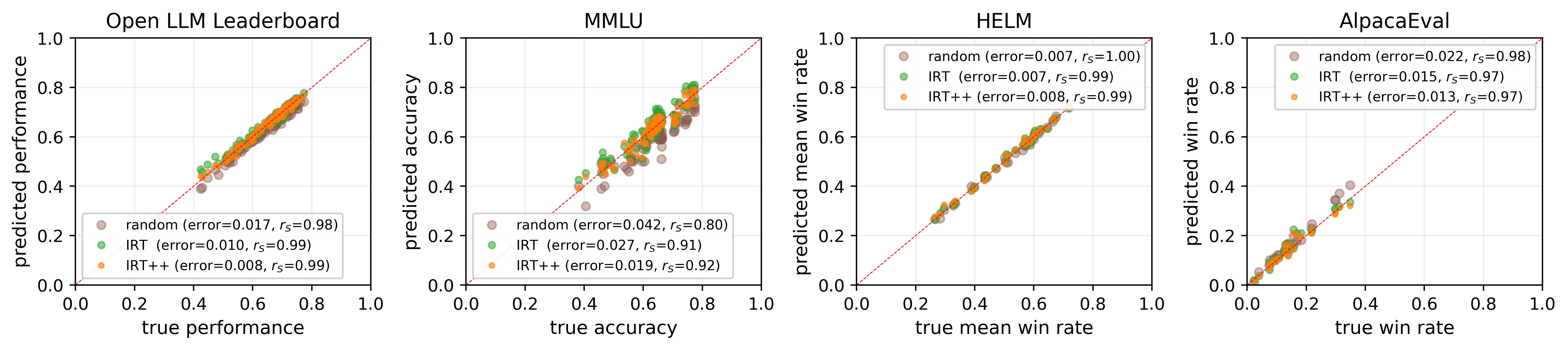}
\vspace{-0.75cm}
\caption{Predicted performance compared with true performance for the four benchmarks (columns) and recent LLMs. We verify the efficacy of the evaluation strategies (IRT and IRT++) we chose to construct tinyBenchmarks.}
\label{fig:leaderboard_performance_individual}
\end{figure*}

\paragraph{Key findings} We investigate the effectiveness of strategies as we increase the number of examples available for evaluating test LLMs. Results for both train-test split scenarios are presented in Figure \ref{fig:leaderboard_performance_acc} (see also Figure \ref{fig:leaderboard_performance_rank} for Spearman's rank correlations). Our main conclusions are:
\begin{itemize}
    \item Our approach to reducing evaluation costs is \emph{effective}. The best-performing strategies achieve estimation error within 2\% on all benchmarks with 100 examples or less per dataset or scenario. For example, for MMLU this reduces the evaluation cost by a factor of 140 (from 14k to 100). For Open LLM Leaderboard even 30 examples per scenario is enough, reducing the evaluation cost by a factor of 160 (from 29K to 180).
    \item  Most strategies perform well when there is a temporal shift between the train and test LLM's (see the lower row of plots in Figure \ref{fig:leaderboard_performance_acc} for the results with ``by date'' split). Thus our approaches for reducing evaluation costs remain \emph{practical} when evaluating the performance of newer, more capable LLMs and can help save GPU hours when evaluating future LLMs and/or checkpoints during pre-training.
    \item \emph{IRT-based methods} (``IRT'' and ``IRT++'') perform consistently well across benchmarks and train-test splits. The gp-IRT (``++'') variation always improves or matches its vanilla counterpart, while adding only a few seconds to the evaluation time (see Figure \ref{fig:running_item_inference}). Thus we use the IRT-based anchor examples to construct \href{https://huggingface.co/tinyBenchmarks}{tiny versions} tiny versions (100 examples per scenario) of each of the benchmarks and release them along with the gp-IRT tool (code and pre-trained IRT model) for efficient evaluation of future LLMs. We present additional evaluations of tinyBenchmarks in Figure \ref{fig:leaderboard_performance_individual} for one of the 5 random seeds in which the random sampling underperforms. In Appendix \ref{sec:mmlu}, we conduct an exploratory analysis of the examples comprising tinyMMLU.
\end{itemize}

\begin{figure}[h]
\centering
\includegraphics[width=\linewidth]{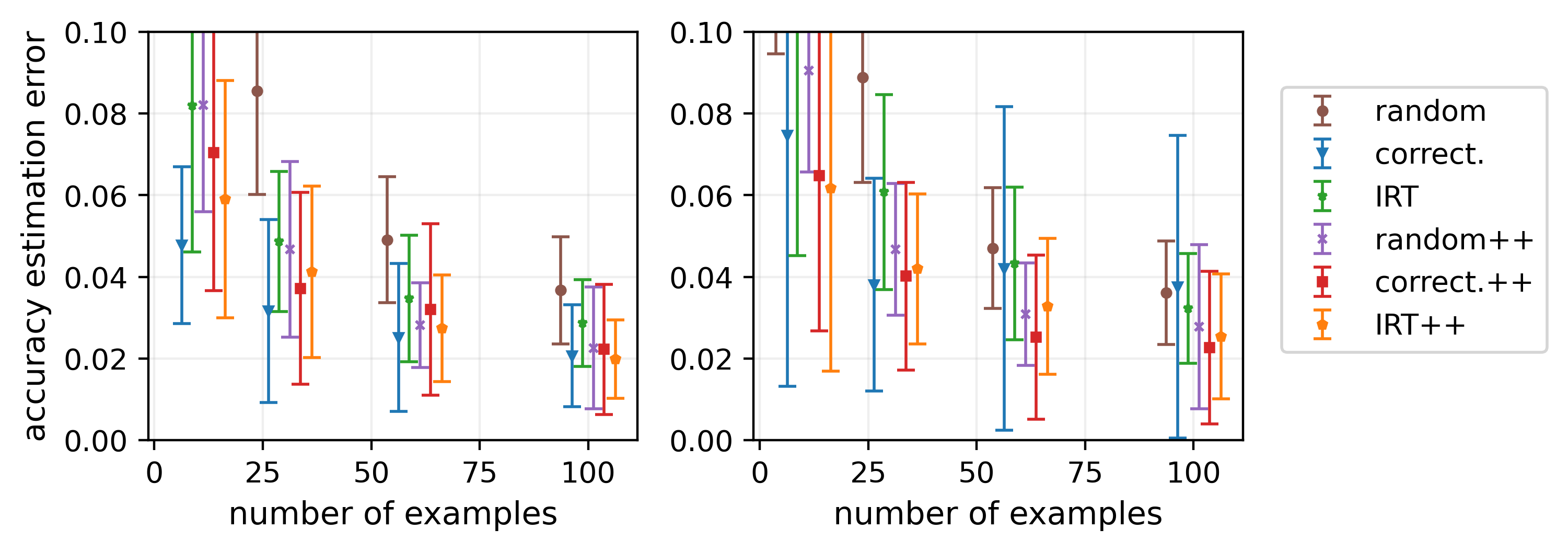}
\vspace{-0.75cm}
\caption{Estimation error on specialized LLMs (right) compared to error on random LLMs (left) on MMLU. Correctness-based example selection is affected the most by this distribution shift.}
\label{fig:mmlu_performance_specialized_models_acc}
\end{figure}

\paragraph{Specialized LLMs} In our previous experiments the test set of LLMs consisted of either a random subset of models or the most recent ones. Both of these test sets are dominated by base and instruction-tuned LLMs. Here we assess the ability of the considered strategies to predict the performance of specialized LLMs, \ie, models fine-tuned for specific domains such as code, biology, or finance. We consider MMLU benchmark and collect a new hand-picked test set of 40 specialized models. Such models are likely to have unique strengths and perform well in specific MMLU categories while relatively underperforming on others. Thus, their correctness patterns might be different from those in the train set, posing a challenge for our evaluation strategies. We present results in Figure \ref{fig:mmlu_performance_specialized_models_acc}. 

As we anticipated, the correctness-based anchor strategy deteriorates when tested on specialized LLMs. In contrast to the IRT-based anchors that are only slightly affected, demonstrating their robustness and supporting our choice to use them for tinyBenchmarks construction.


\begin{figure}[h]
\centering
\includegraphics[width=.75\linewidth]{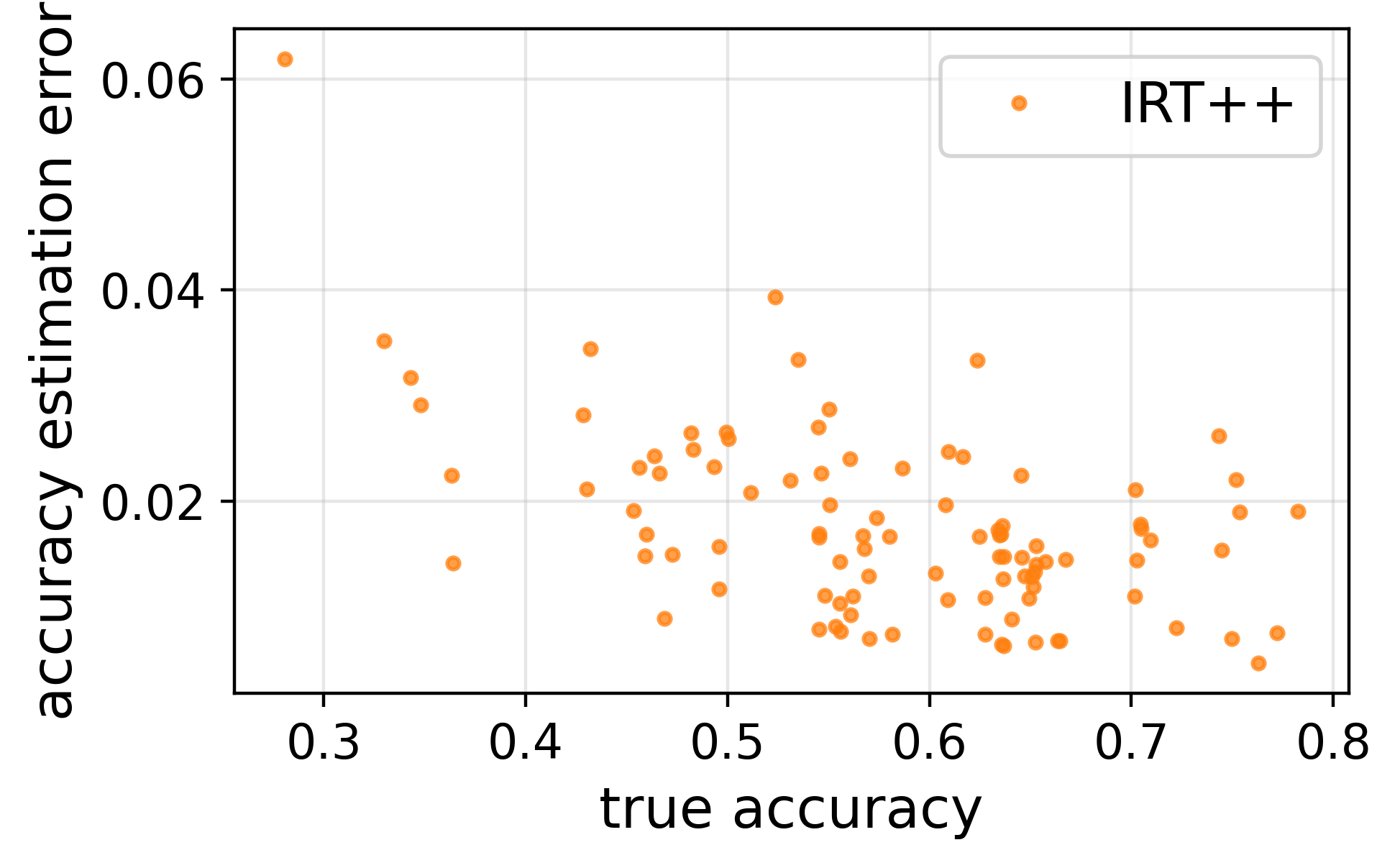}
\vspace{-0.5cm}
\caption{Spread of estimation errors across a random subset of LLMs with varying capabilities on MMLU. The error tends to be slightly lower for more capable models. The worst case error across almost all models is $\leq 4\%$.}
\label{fig:performance2_mini_mmlu}
\end{figure}

\paragraph{Estimation error analysis} We present a more detailed view of the estimation error of the best performing ``IRT++'' evaluation strategy on MMLU with 100 examples. In Figure \ref{fig:performance2_mini_mmlu} we plot estimation error against the actual accuracy of 99 test LLMs for a random train-test split. Our strategy can estimate the performance of more capable LLMs slightly better, although there is no strong dependency. We also note that the estimation error never exceeds 4\% (except for one LLM with extremely low performance). Recall that the average error is 2\% as shown in Figure \ref{fig:leaderboard_performance_acc}, supporting the reliability of our evaluation approach.







\begin{figure*}[ht!]
\centering
\includegraphics[width=\linewidth]{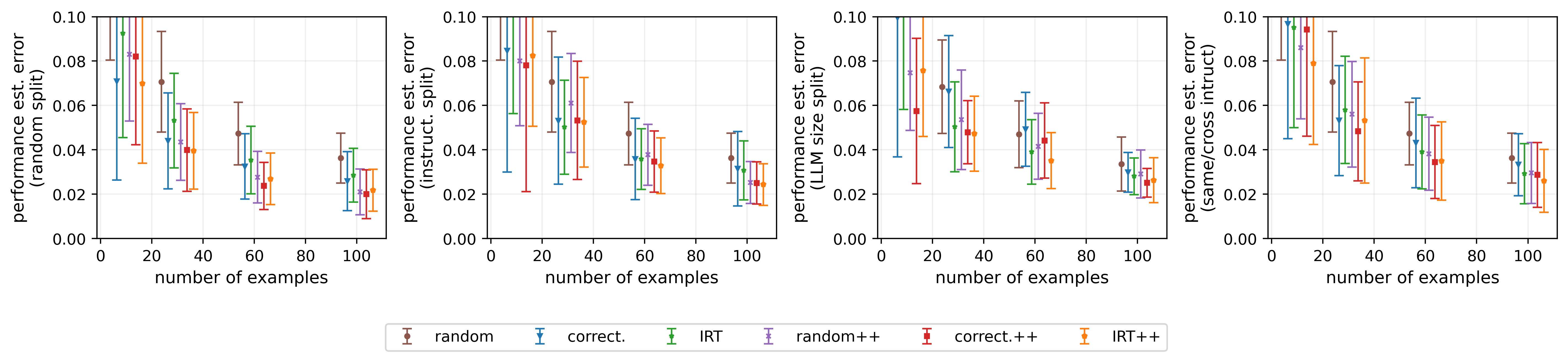}
\vspace{-0.75cm}
\caption{Estimation error when predicting the performance of prompt templates. The results demonstrate that using our methods for efficient prompt-based model evaluation is a promising application.}
\label{fig:icl_templates}
\end{figure*}

\section{Conclusion}
\label{sec:concl}

In this paper, we demonstrate it is possible to accurately assess the capabilities of LLMs with a fraction (sometimes two orders of magnitude smaller) of the examples in common benchmark datasets by leveraging models of educational assessments from psychometrics. This leads directly to savings in terms of the monetary costs associated with evaluating LLMs, but also the computational and environmental costs. For practitioners, the computational cost savings are especially convenient because they enable them to evaluate LLMs more frequently during fine-tuning and prompt engineering.

Based on our results we are releasing tinyBenchmarks, pre-selected subsets of examples from the widely adopted LLM benchmarks. tinyBenchmarks are simply small datasets that are straightforward to use to evaluate LLMs cheaply. We are also releasing an IRT-based tool to enhance performance estimation. The tool provides code and IRT parameters trained on the corresponding benchmarks and can be run on a CPU in a few seconds.

\subsection{Extensions}

\paragraph{Prompt evaluation} A persistent challenge in prompt-based model evaluation is the influence the prompting setup has on model predictions \citep[see, \eg,][]{lu2021fantastically, mishra2021reframing, min2022rethinking, kim2022ground, weber2023mind, wei2023larger}.
We can use the previously described approaches to make predictions across different prompting setups. This way, we can estimate how well a model will do on a new set of prompts using just a few evaluations, or how a new model will perform on a given prompt.
To test this idea, we train an IRT model on the prediction data from \citet{weber2023icl}, containing evaluations of eight LLaMA LLMs \citep[vanilla or instruction tuned on the Alpaca self-instruct dataset;][]{touvron2023llama,alpaca2023selfinstruct} for the ANLI dataset \citep{nie2020adversarial}.
The dataset consists of evaluations of the 750 data points wrapped with 15 different instruction templates sourced from the promptsource collection \citep[P3;][]{bach2022promptsource}.

Similarly to our previous experiments, we evaluate random splits and splits featuring distribution shifts (across model sizes and different instruction templates). For model size, we put all models with sizes 7B, 13B, and 30B in the training set while the models with size 65B go to the test set. For splits related to prompts templates, we consider two different approaches: first, we conduct a 2-fold cross-validation rotating instruction templates; second, we consider using the same and different instruction templates in the in-context-learning examples and in the input example alternating the strategies in the training and test sets. Results in Figure \ref{fig:icl_templates} suggest that prompt-based model evaluation can be efficiently carried out with the methods introduced in this work, even in the presence of several practical distribution shifts.

\paragraph{Adaptive testing} We expect further performance estimation improvements can be squeezed out by more sophisticated applications of similar ideas. For example, instead of pre-selecting a subset of examples before evaluating the LLM, it may be possible to select the examples \emph{adaptively} during the evaluation process. This idea is widely used in the computerized-assisted testing algorithms behind many standardized tests. We demonstrate preliminary results on MMLU using an adaptive IRT variant in Figure \ref{fig:adaptive_testing_mmlu} (see Figure \ref{fig:leaderboard_performance_acc_adaptive_sampling} for results on more benchmarks). Although the estimation performance has improved, our current implementation takes over 5 minutes to run, which might not be as appealing practically.

\begin{figure}[h]
\centering
\includegraphics[width=.85\linewidth]{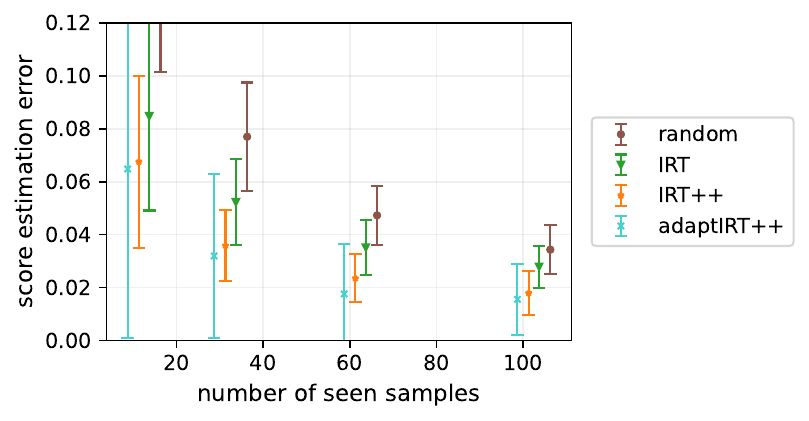}
\vspace{-0.5cm}
\caption{Preliminary adaptive testing results on MMLU.}
\label{fig:adaptive_testing_mmlu}
\end{figure}

\subsection{Limitations}
The main limitations of the methods described in this paper are related to potential severe distribution shifts. Taking MMLU as an example, we anticipate larger performance estimation errors for models that fail on simple questions while answering complicated ones correctly, thus altering the correctness patterns. This might be caused by significant architecture or pre-training data changes. A rapid increase in LLM capabilities may also cause extrapolation errors. To alleviate these problems, we recommend periodically updating the curated examples and IRT parameter estimates using data from more modern LLMs.


\section*{Acknowledgements}

We are grateful for the help provided by Yotam Perlitz in downloading data from HELM. This paper is based upon work supported by the National Science Foundation (NSF) under grants no.\ 2027737 and 2113373.

\section*{Impact Statement}
This paper presents work whose goal is to advance the field of Machine Learning. There are many potential societal consequences of our work, none which we feel must be specifically highlighted here.

\bibliography{FMP, MY}
\bibliographystyle{icml2024}

\newpage
\appendix
\onecolumn

\section{Evaluation when subscenarios have different number of samples}\label{sec:diff_number_samples}

Suppose we want to estimate the performance of a scenario $j$ which is composed of $s_j$ subscenarios. Denote the set of examples in each subscenario of $j$ as $\cI_{jk}$, for $k\in\{1,\cdots,s_j\}$. Then, $\cI_j=\cup_k\cI_{jk}$, with disjoint $\cI_{jk}$'s. For a given LLM $l$, our main goal is then to estimate $\frac{1}{s_j}\sum_{k}\frac{1}{|\cI_{jk}|}\sum_{i \in \cI_{jk}} Y_{il}$. See that we can write
\begin{talign*}
     \frac{1}{s_j}\sum_{k}\frac{1}{|\cI_{jk}|}\sum_{i \in \cI_{jk}} Y_{il}=\sum_{k}\sum_{i \in \cI_{jk}} \frac{1}{s_j|\cI_{jk}|}Y_{il}=\sum_{i \in \cI_{j}} \bar{\omega}_i Y_{il}.
\end{talign*}
This tells us that we can represent the performance of model $l$ as a weighted average instead of a simple average. In our code, $\omega_i\triangleq|\cI_j|\cdot\bar{\omega}_i$'s are called \texttt{balance\_weights} and $\bar{\omega}_i$'s are called \texttt{normalized\_balance\_weights}. In Section \ref{sec:strategies}, when computing the estimates using the stratified random sampling strategy, the weights for each example are still given by $1/|\hat{\cI}_j|$ (because subscenarios should already be equally represented) but when using the clustering ideas, the weight for each anchor point is given by the sum of $\bar{\omega}_i$'s of all items in its cluster. We do not apply any weighting when fitting the IRT models but only when computing the p-IRT (and gp-IRT) estimate:
\begin{talign*}
    \hat{Z}^{\text{p-IRT}}_{jl} =\frac{\hat{\lambda}}{|\widehat{\cI}_j|}\sum_{i\in \widehat{\cI}_j} \omega_iY_{i l}+\frac{1-\hat{\lambda}}{|{\cI}_j\setminus\widehat{\cI}_j|}\sum_{i \in {\cI}_j\setminus\widehat{\cI}_j  } \omega_i\hat{p}_{il}\nonumber.
\end{talign*}

\section{tinyMMLU}\label{sec:mmlu}

To construct tinyMMLU we chose 100 examples and weights identified by the IRT anchor point approach (``IRT'') corresponding to the best test performance (across random seeds) in the experiment presented in the top part of Figure \ref{fig:leaderboard_performance_acc} on MMLU. For comparison, we analogously selected 100 examples with the correctness anchor point method.

To better understand the composition of tinyMMLU, in Figure \ref{fig:w_mini_mmlu} we visualize the distribution of the weights of the selected examples and compare it to the weights of the correctness anchors. Recall that weights are non-negative and sum to 1. If an item has a weight $0.1$, for example, that item has a contribution of $10\%$ in the final estimated score.  From Figure \ref{fig:w_mini_mmlu}, we can see that tinyMMLU has more uniform weights compared to its correctness-based counterpart. We measure uniformity through the effective sample size (ESS) of the example weights. ESS, traditionally used in the Monte Carlo and domain adaptation \cite{elvira2022rethinking,maia2023effective} literature, measures weight inequality in a way such that $\text{ESS}=0.50$, for example, informally means that the corresponding weighted average is influenced by only $50\%$ of (uniformly weighted) examples. In the context of our problem, more uniform weights of tinyMMLU contribute to its robustness when evaluating LLMs with varying correctness patterns, such as specialized LLMs in Figure \ref{fig:mmlu_performance_specialized_models_acc}.


We also investigate the total weight of the tinyMMLU examples within each of the 57 subjects in Figure \ref{fig:w_sub_mini_mmlu}. The highest weighted are ``high school psychology'', ``elementary mathematics'', and ``professional law''. Interestingly the weight of the subjects is fairly different from its correctness-based counterpart.

\begin{figure}[h]
\centering
\includegraphics[width=.35\textwidth]{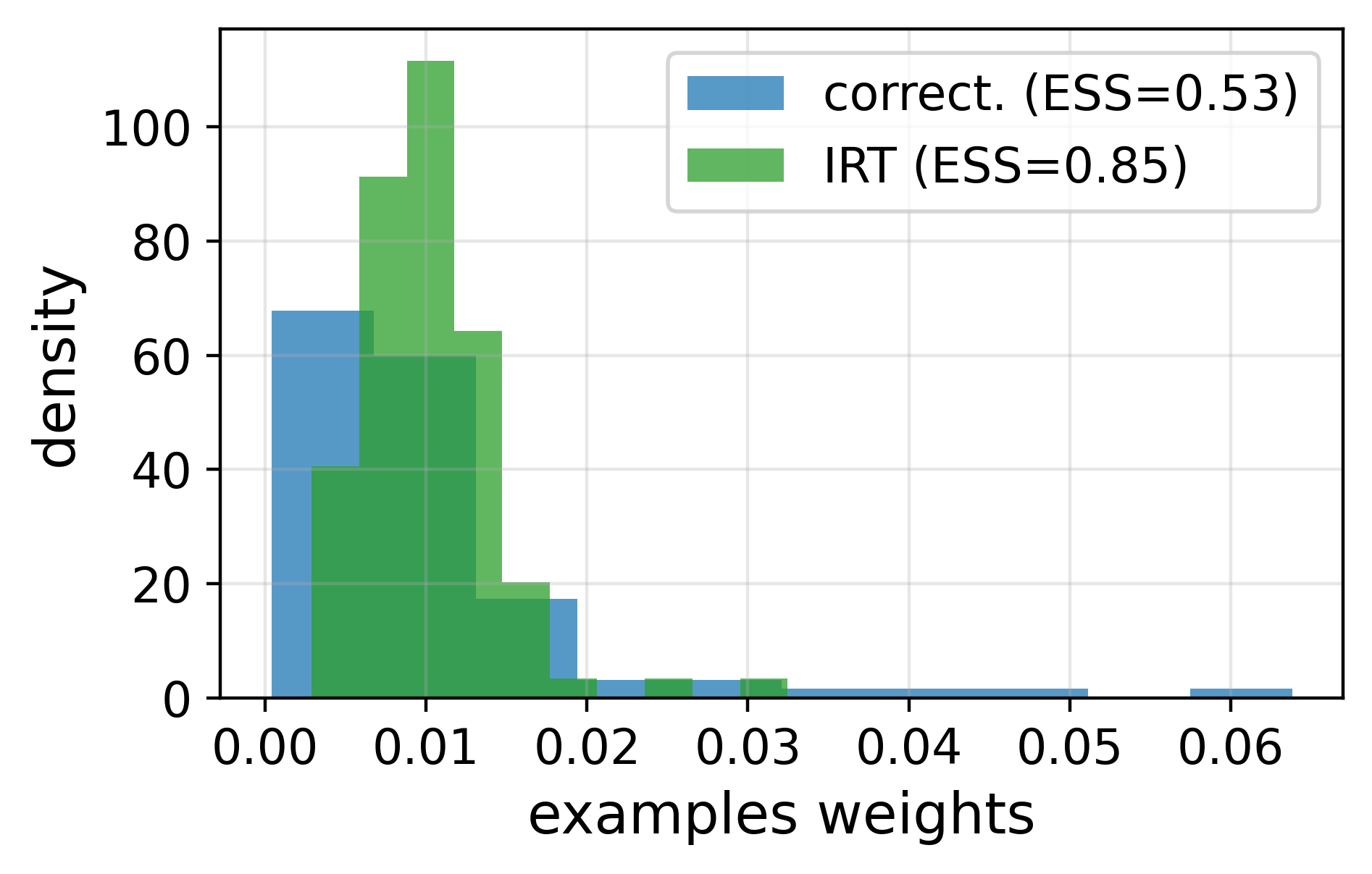}
\caption{Comparing the spread of examples weights using both the IRT and correctness approaches to find anchor points. We see that weights inequality is much higher when we cluster examples using correctness.}
\label{fig:w_mini_mmlu}
\end{figure}

\begin{figure*}[h]
\centering
\includegraphics[width=.75\textwidth]{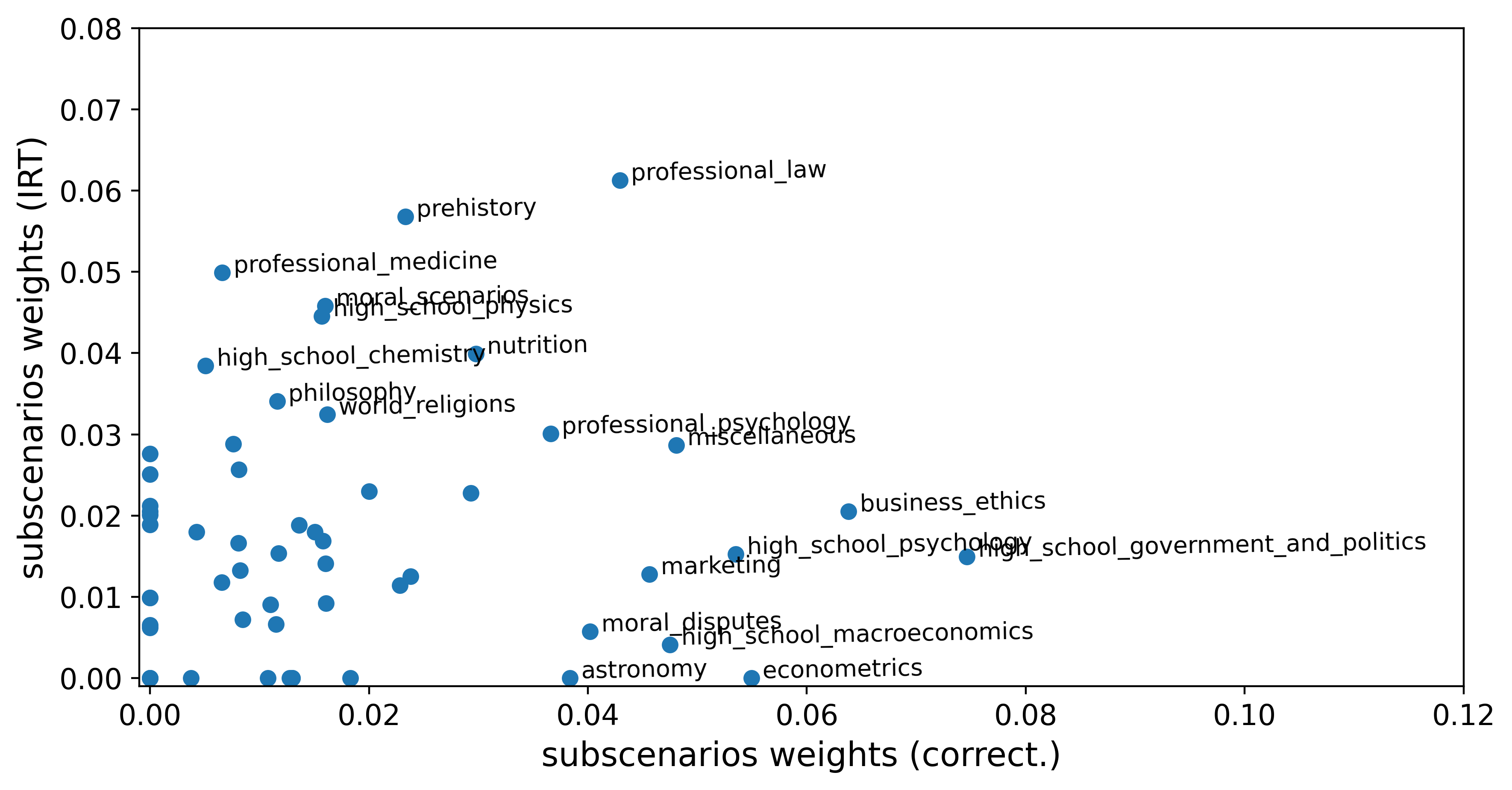}
\caption{Weights given to MMLU subscenarios by the two anchoring methods.}
\label{fig:w_sub_mini_mmlu}
\end{figure*}
\section{Proof of Proposition \ref{prop:pirt-consistency}}
\begin{proof}[Proof of proposition \ref{prop:pirt-consistency}]
See that
    \begin{talign*}
        |\widehat{\Ex}[Z_{jl} \mid Y_{i_0 l}, \cdots, Y_{i_k l}]-\Ex[Z_{jl} \mid Y_{i_0 l}, \cdots, Y_{i_k l}]|&\leq \frac{1-\hat{\lambda}}{|{\cI}_j\setminus\widehat{\cI}_j|}\sum_{i \in {\cI}_j\setminus\widehat{\cI}_j  }| \sigma(\hat{\theta}_l^\top \alpha_i-\beta_i)-\sigma(\theta_l^\top \alpha_i-\beta_i)|\\
        &\leq \frac{1}{|{\cI}_j\setminus\widehat{\cI}_j|}\sum_{i \in {\cI}_j\setminus\widehat{\cI}_j  }|(\hat{\theta}_l-\theta_l)^\top \alpha_i|\\
        &\leq \frac{1}{|{\cI}_j\setminus\widehat{\cI}_j|}\sum_{i \in {\cI}_j\setminus\widehat{\cI}_j  }\norm{\alpha_i}_2\norm{\hat{\theta}_l-\theta_l}_2 \\
        &\leq c \norm{\hat{\theta}_l-\theta_l}_2 \to 0
    \end{talign*}
    in probability as $|\widehat{\cI}|\to\infty$. The second step uses the fact that $\sigma$ is 1/4-Lipschitz and the third step applies Cauchy-Schwarz inequality.
\end{proof}

\section{More details about benchmarks}\label{sec:benchmark}

\begin{itemize}
    \item HuggingFace's Open LLM Leaderboard \citep{open-llm-leaderboard}: the data from this benchmark is composed of 395 LLMs and approx. 29k items that were downloaded from the platform in January/2024. To extract data from those models, we filter all models from the platform that have an MMLU score over\footnote{On the leaderboard. The actual score we use can be different because we use the last submission to the leaderboard, while the leaderboard shows the best results among all submissions.} $.3$, order them according to their average performance, and equally spaced selected models. Then, we kept all models that had scores for all six scenarios: ARC \citep{clark2018think}, HellaSwag \citep{zellers2019hellaswag}, MMLU \citep{hendrycks2020measuring}, TruthfulQA \citep{lin2021truthfulqa}, Winogrande \citep{sakaguchi2021winogrande}, and GSM8K \citep{cobbe2021training}. In a second round of data collection, we collected data for 40 ``specialized models" by recognizing which models were fine-tuned to do the math, coding, \etc. The two sets of models have an intersection, and in total, we have collected data from 428 LLMs.

    \item HELM \citep{liang2022holistic}: we use HELM Lite (\url{https://crfm.stanford.edu/helm/lite}) v1.0.0, which is a dataset composed of 37 LLMs and approx. 10k evaluation examples from 10 scenarios. The scenarios are OpenbookQA \citep{mihaylov2018can}, MMLU \citep{hendrycks2020measuring}, NarrativeQA \citep{kovcisky2018narrativeqa}, NaturalQuestions (closed-book) \citep{kwiatkowski2019natural}, NaturalQuestions (open-book), Math \citep{hendrycks2021measuring}, GSM8K \citep{cobbe2021training}, LegalBench \citep{guha2024legalbench}, MedQA \citep{jin2021disease}, WMT14 \citep{bojar2014findings}.
\end{itemize}

\section{Extra results}

\subsection{Robustness in predicting performance in a longer time horizon}

We conduct extra ablation studies placing 75\% of the data in the test set. For the Open LLM Leaderboard and MMLU, it means we are using 3 months of future data as the test set (vs. approx. 3 weeks in the main text) while for AlpacaEval 2.0 that would correspond to 6 months (vs. approx. 2 months in the main text). In general, we show that our main method ``IRT++'' is pretty robust to the advancements in the field when predicting the performance of new LLMs. We report in the following plots the average estimation error in the test set (using 75\% of the most recent data in the test set) and standard deviation across LLMs. The results do not differ considerably from the ones in the main text.

\begin{figure}[H]
\centering
\includegraphics[width=1\textwidth]{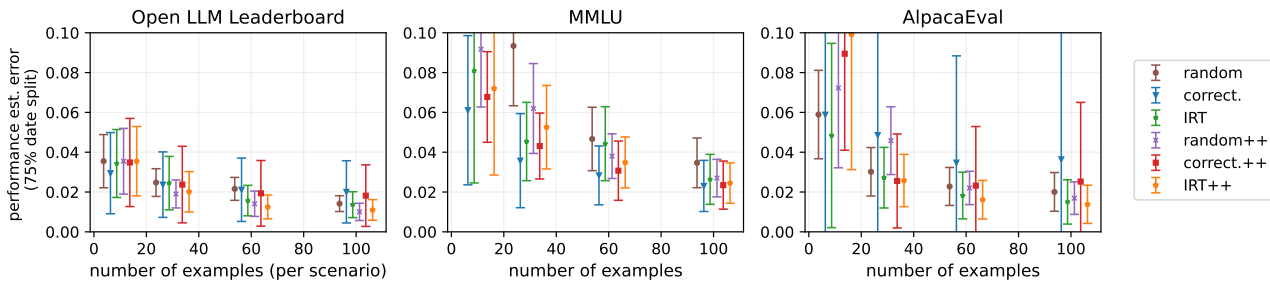}
\caption{Our methods are robust in predicting performance in a longer time horizon}
\end{figure}

\subsection{How costly is it for stratified random sampling beat IRT++ with larger samples?}

We present results comparing IRT++ and stratified random sampling for a larger number of evaluation examples $n$. On Open LLM Leaderboard 400 examples per task (2400 total) are enough to match IRT++ with 100 examples per task (600 total). On MMLU, random sampling improves quite slowly and would require $>$400 examples to match IRT++ at 100. On AlpacaEval, random with 200 examples matches IRT++ with 100 examples (note that AlpacaEval is a small benchmark with 805 examples total, but evaluation requires GPT-4 and is thus quite expensive). We use the random split for the LLMs, implying no distribution shift between train and test.


\begin{figure}[H]
\centering
\includegraphics[width=1\textwidth]{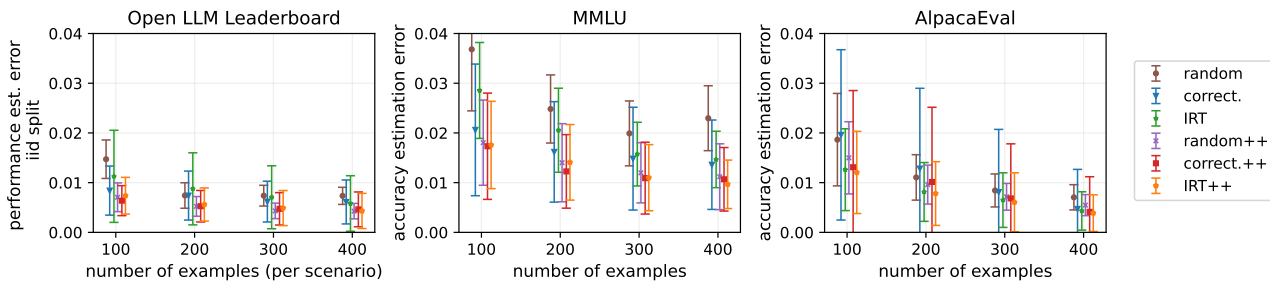}
\caption{Benchmark results for different methods and sample sizes}
\end{figure}

\subsection{Running time}
We record the running time of IRT inference (ability parameter fitting) when running our experiments. In Figure \ref{fig:running_item_inference} we show that the average running time is fairly negligible.

\begin{figure}[H]
\centering
\includegraphics[width=.35\textwidth]{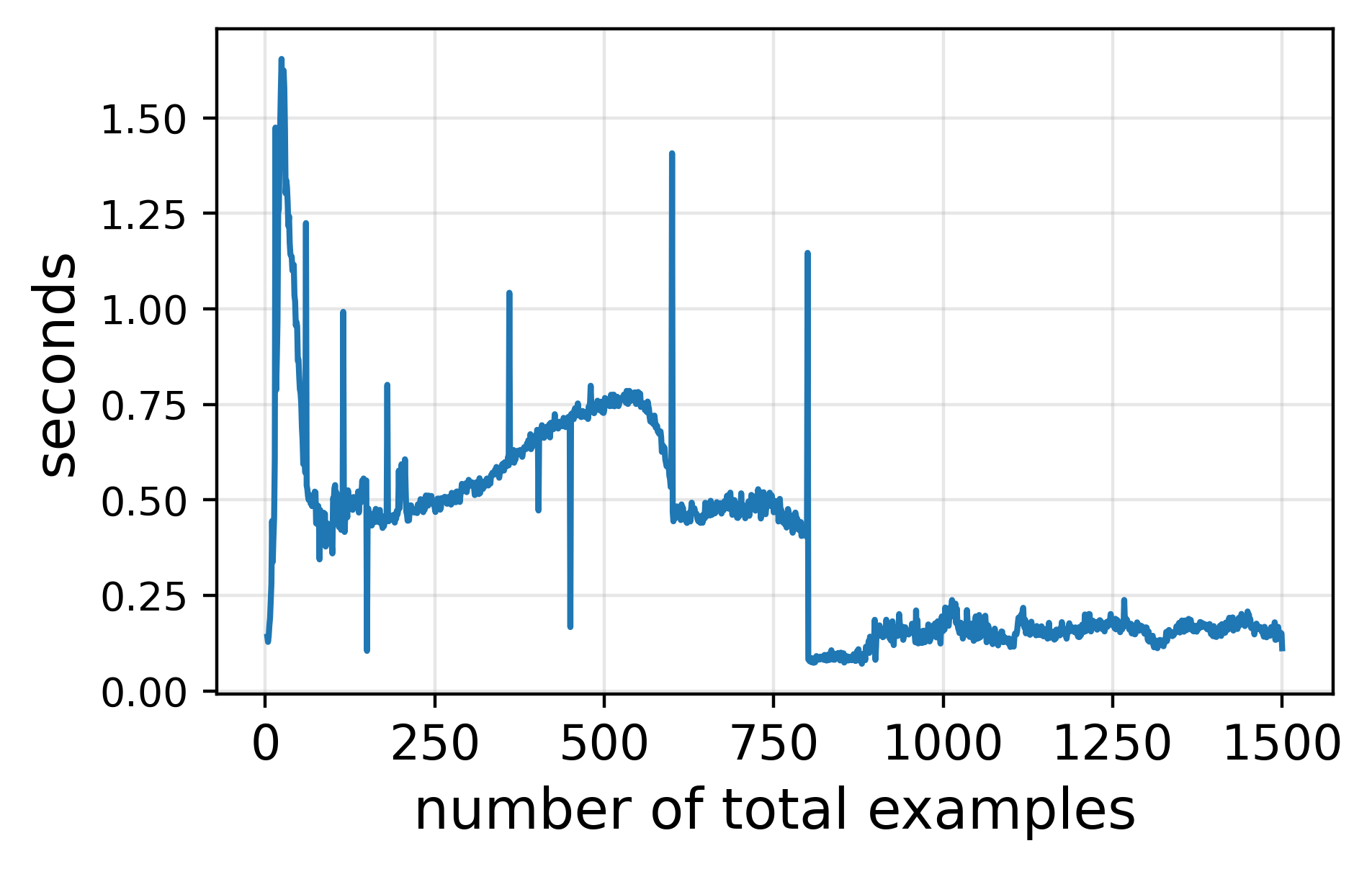}
\caption{Average running time by the amount of test examples: IRT inference.}
\label{fig:running_item_inference}
\end{figure}

\subsection{Rank correlation results}
In this section, we explore versions of Figures \ref{fig:leaderboard_performance_acc} and \ref{fig:mmlu_performance_specialized_models_acc}  when we look at rank correlation (correlation between true and predicted ranking) instead of performance. It is clear from the plots below that our method can be used to rank models efficiently with tiny samples.

\begin{figure}[H]
\centering
\includegraphics[width=.9\textwidth]{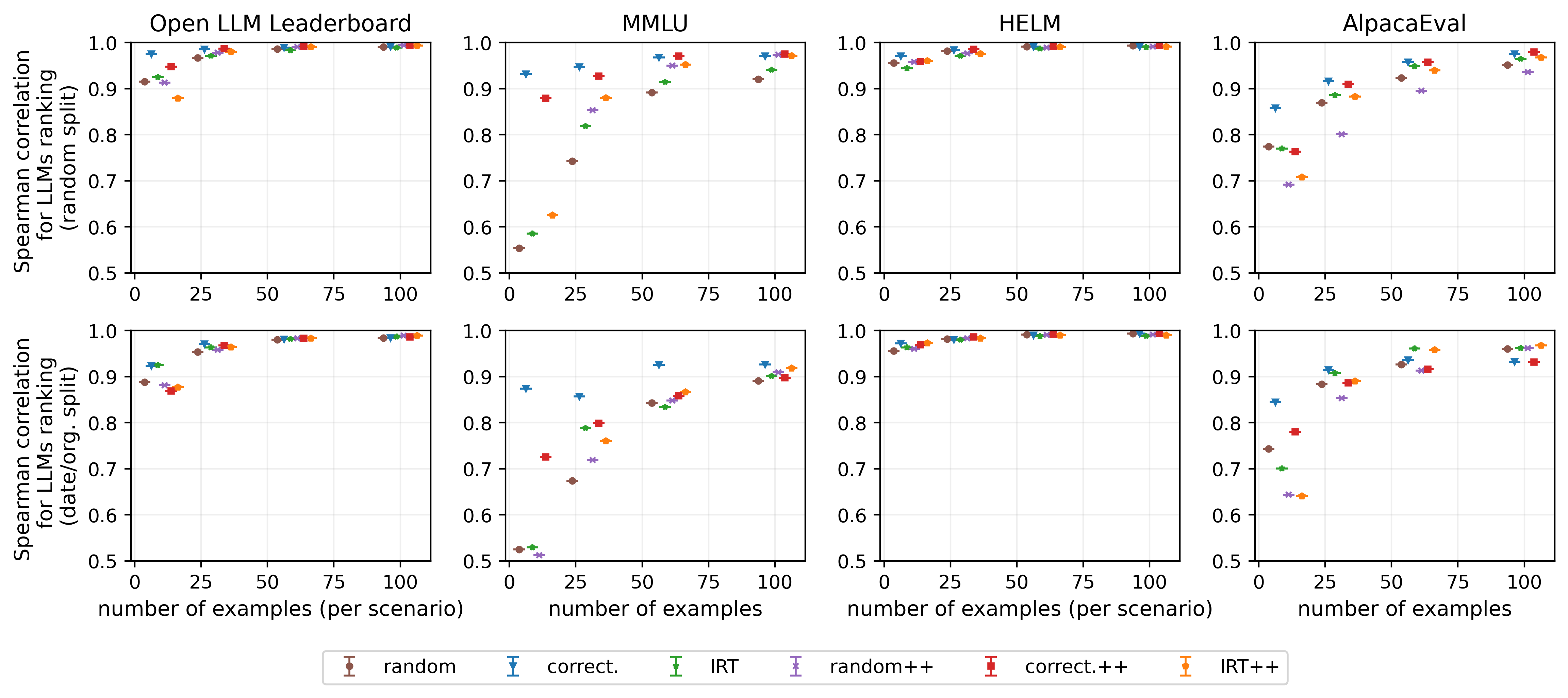}
\caption{Rank correlation for true performance and predicted performance among LLMs.}
\label{fig:leaderboard_performance_rank}
\end{figure}

\begin{figure}[H]
\centering
\includegraphics[width=.55\textwidth]{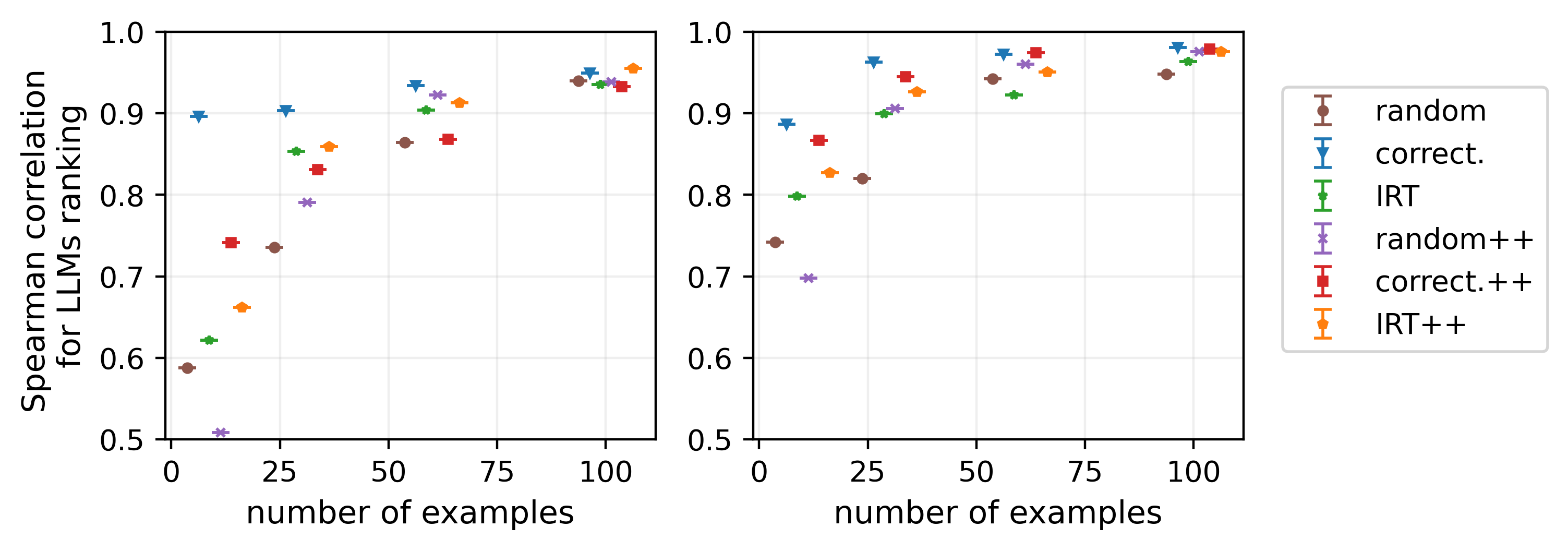}
\caption{Rank correlation for true performance and predicted performance among LLMs in MMLU. The plot on the left represents a random split of the data while the plot on the right considers specialized models as the test set.}
\label{fig:mmlu_performance_specialized_models_rank}
\end{figure}

\subsection{Adaptive testing}
\label{app:adaptive_testing}
In this section, we complement the results shown in Figure \ref{fig:adaptive_testing_mmlu} for all benchmarks.

\begin{figure}[H]
\centering
\includegraphics[width=.9\textwidth]{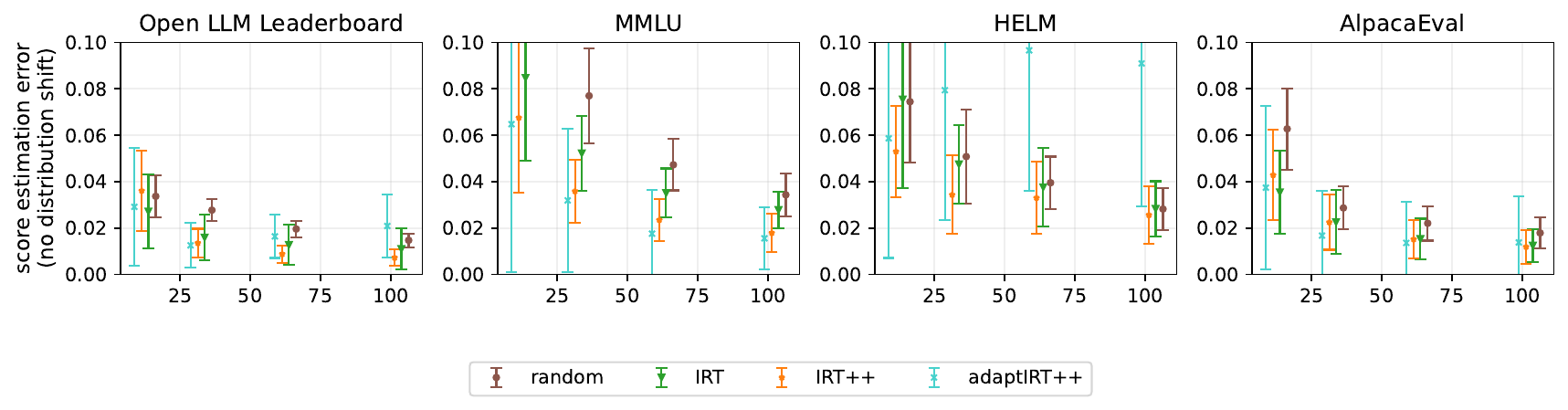}
\caption{Results of adaptive testing for different benchmarks.}
\label{fig:leaderboard_performance_acc_adaptive_sampling}
\end{figure}


\section{Individual performances per scenario}
In this section, we explore what is behind Figure \ref{fig:leaderboard_performance_acc} by looking in detail at results for individual scenarios for the Open LLM Leaderboard and HELM. It is clear from the following plots that there are scenarios in which our methods shine more than others.

\subsection{Open LLM Leaderboard}
\begin{figure}[H]
\centering
\includegraphics[width=.75\textwidth]{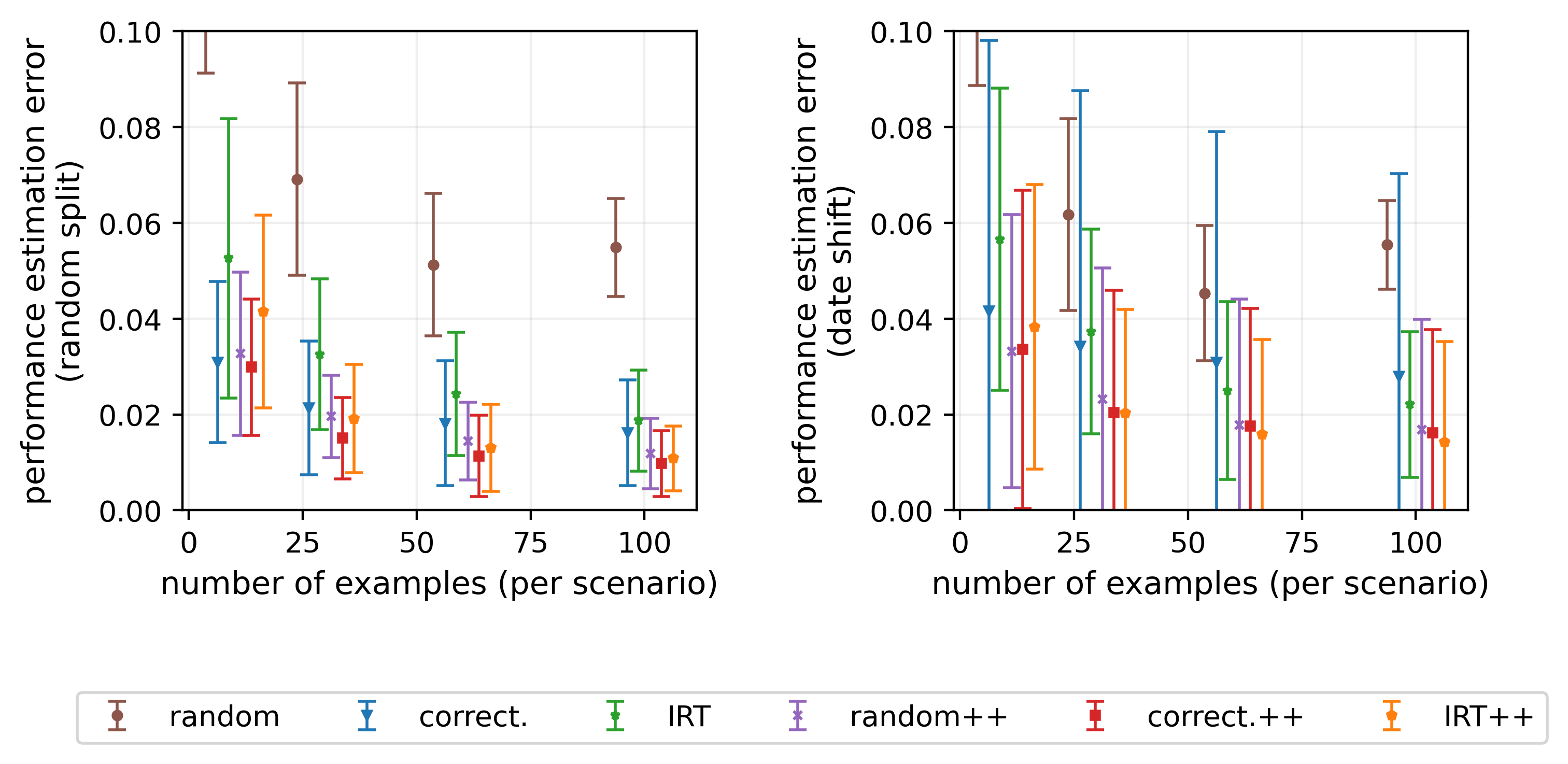}
\caption{ARC}
\end{figure}

\begin{figure}[H]
\centering
\includegraphics[width=.75\textwidth]{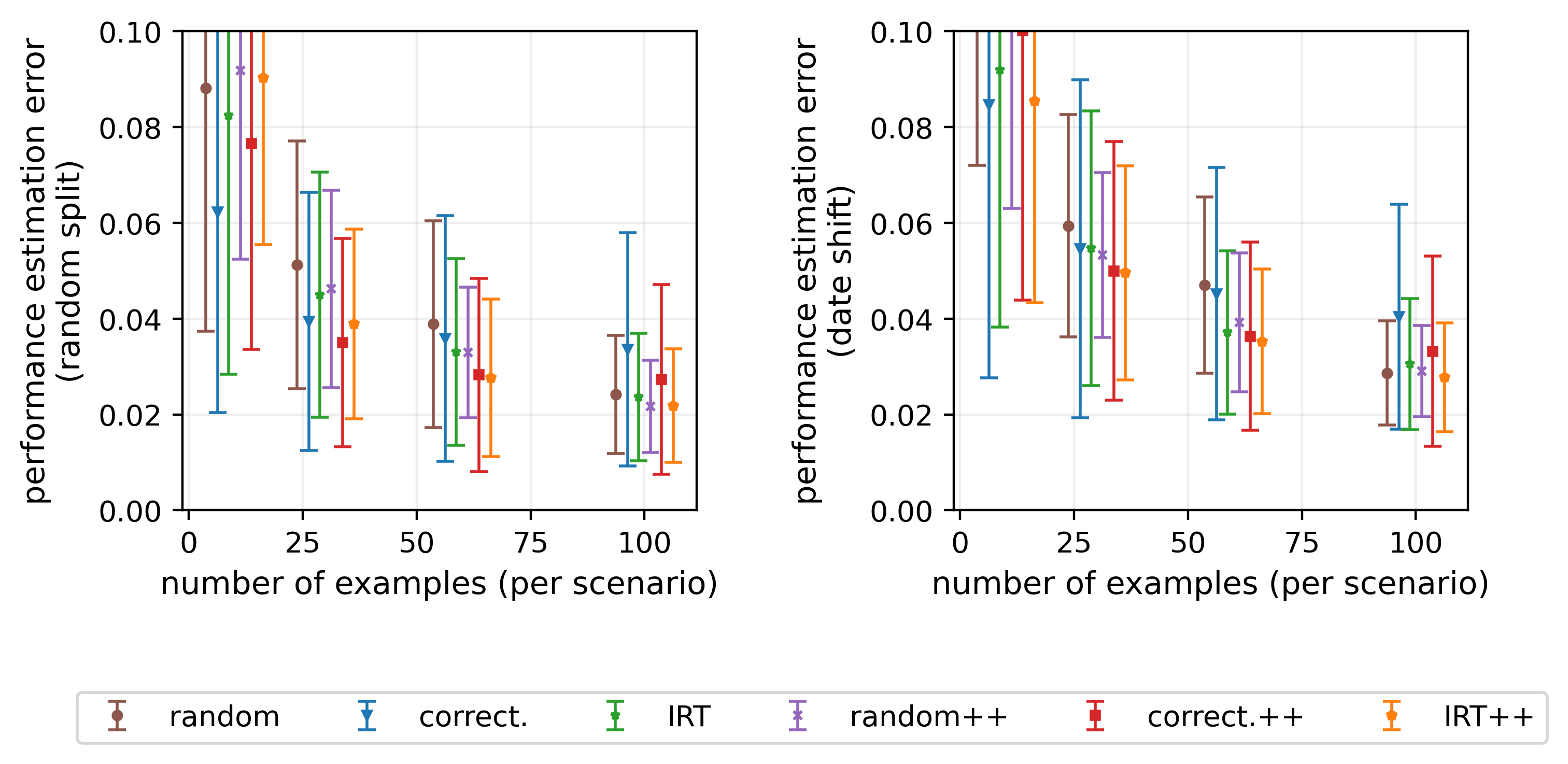}
\caption{GSM8K}
\end{figure}

\begin{figure}[H]
\centering
\includegraphics[width=.75\textwidth]{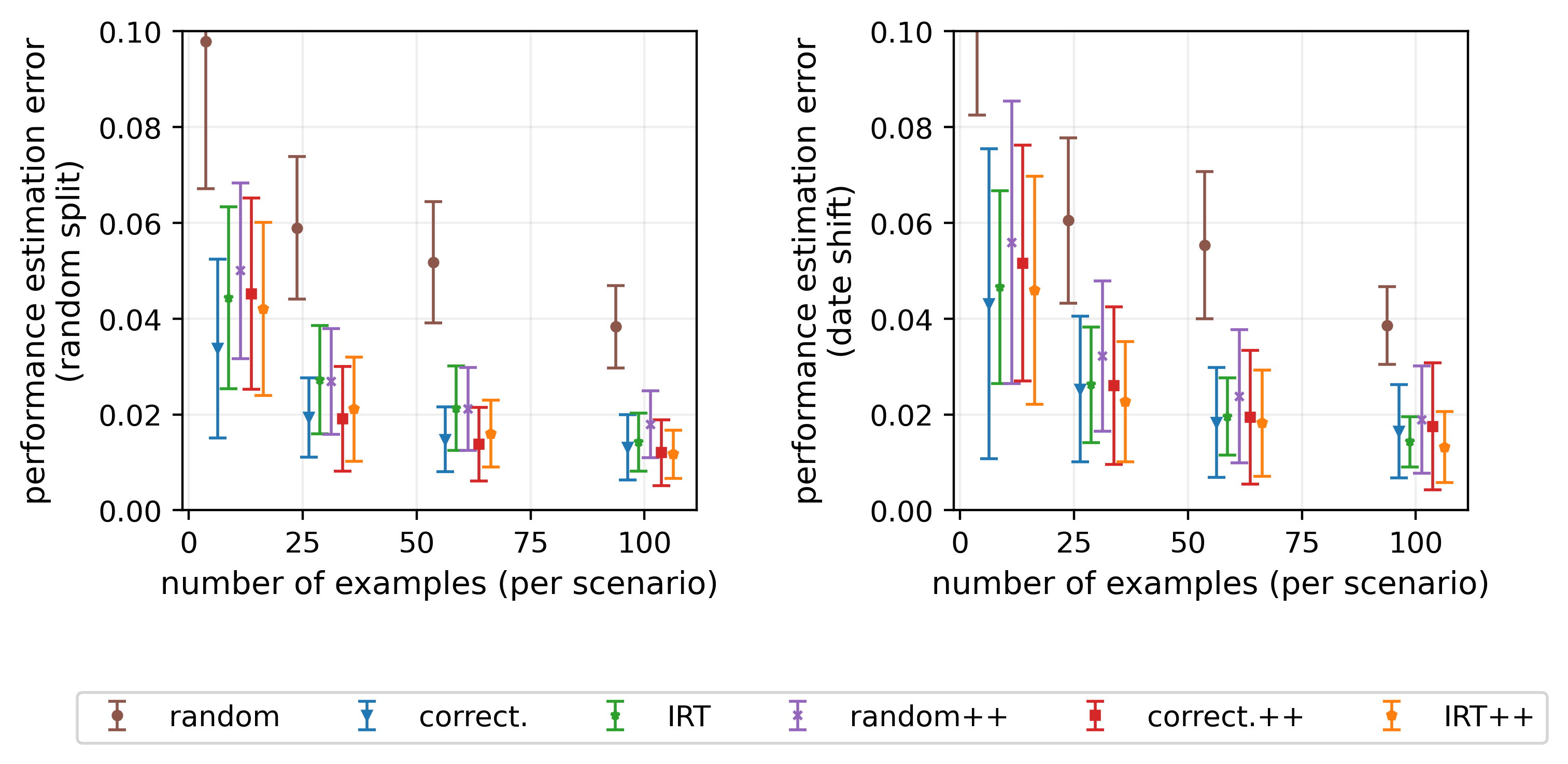}
\caption{TruthfulQA}
\end{figure}

\begin{figure}[H]
\centering
\includegraphics[width=.75\textwidth]{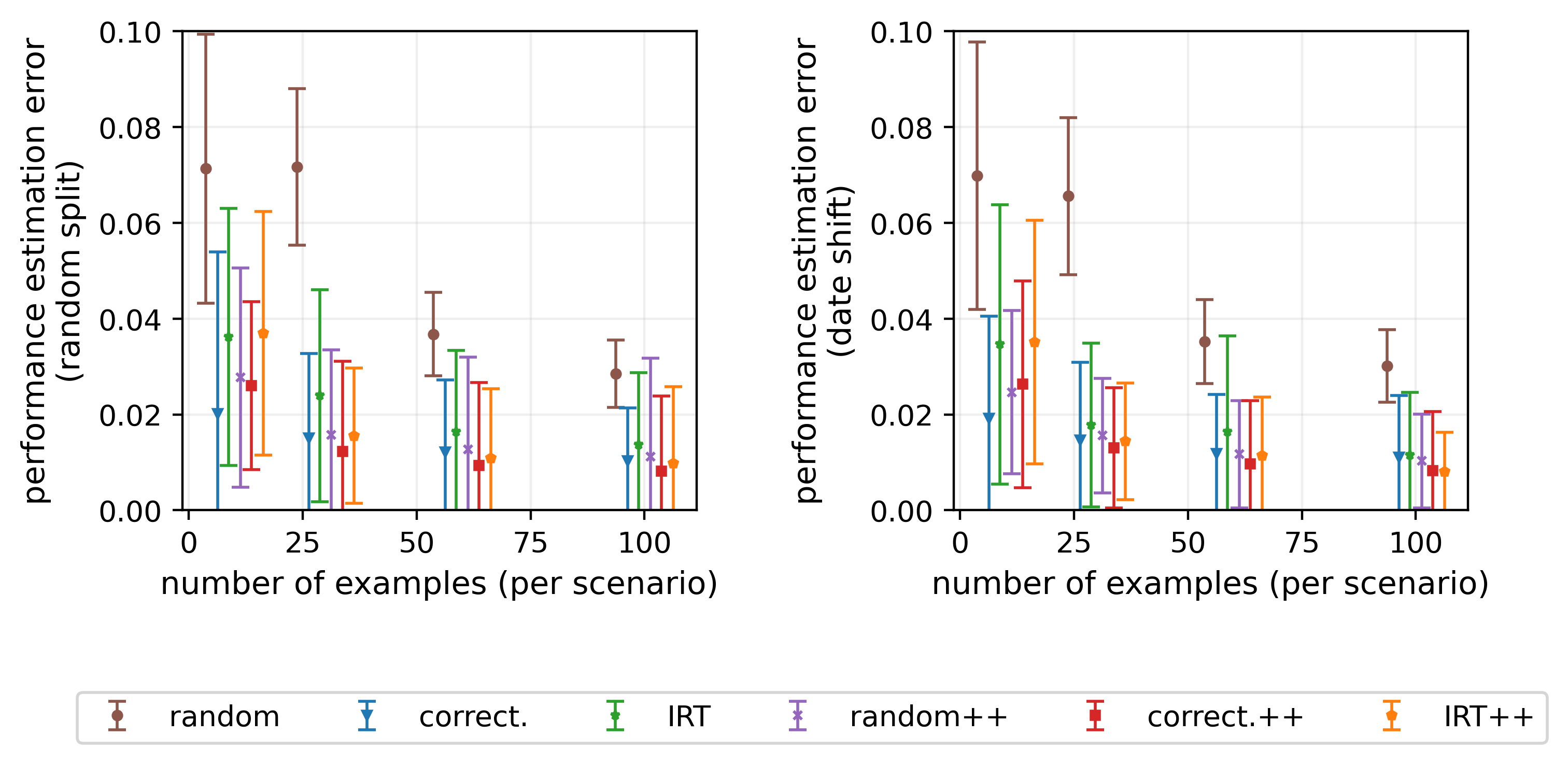}
\caption{HellaSwag}
\end{figure}

\begin{figure}[H]
\centering
\includegraphics[width=.75\textwidth]{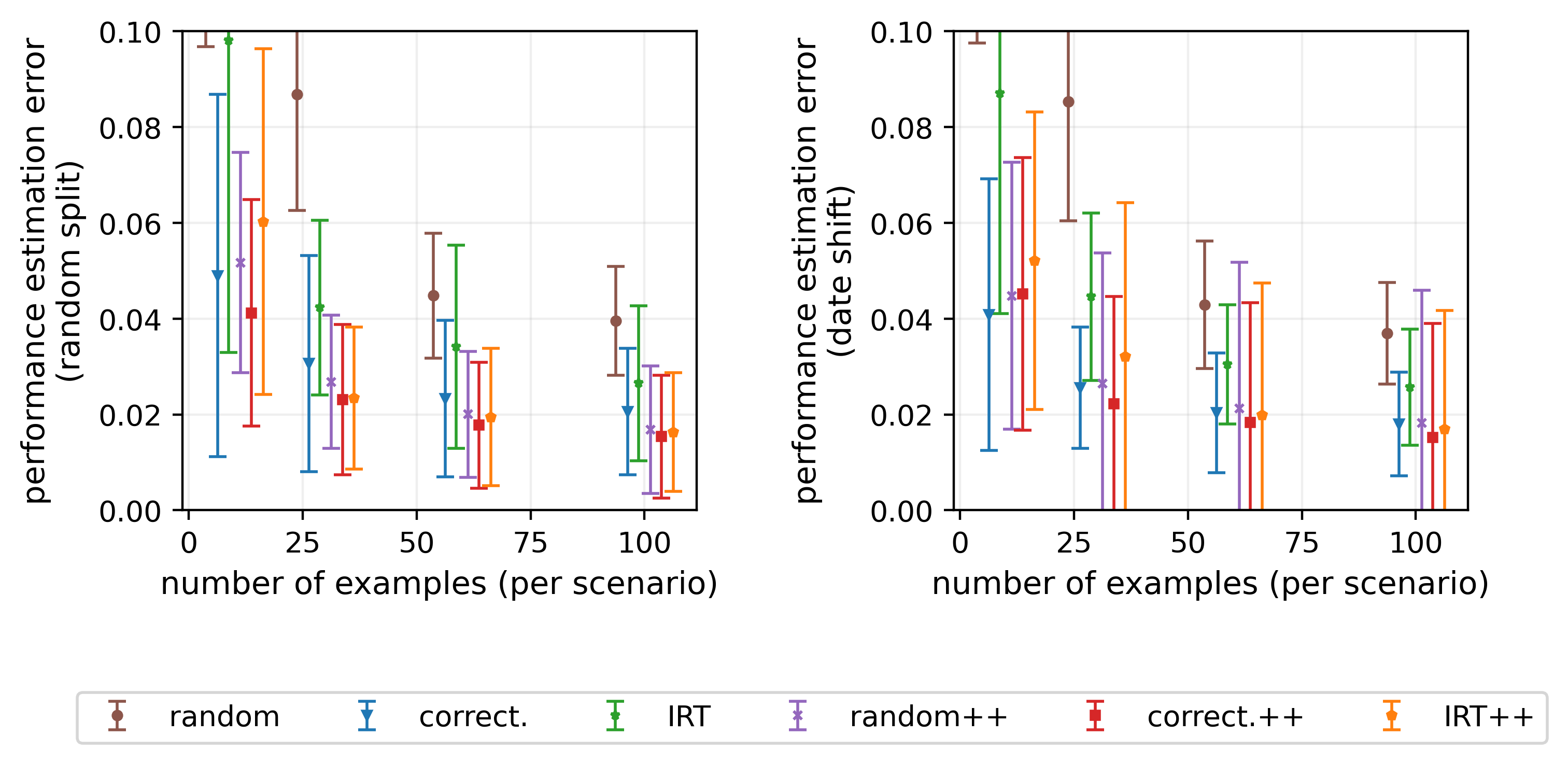}
\caption{MMLU}
\end{figure}

\begin{figure}[H]
\centering
\includegraphics[width=.75\textwidth]{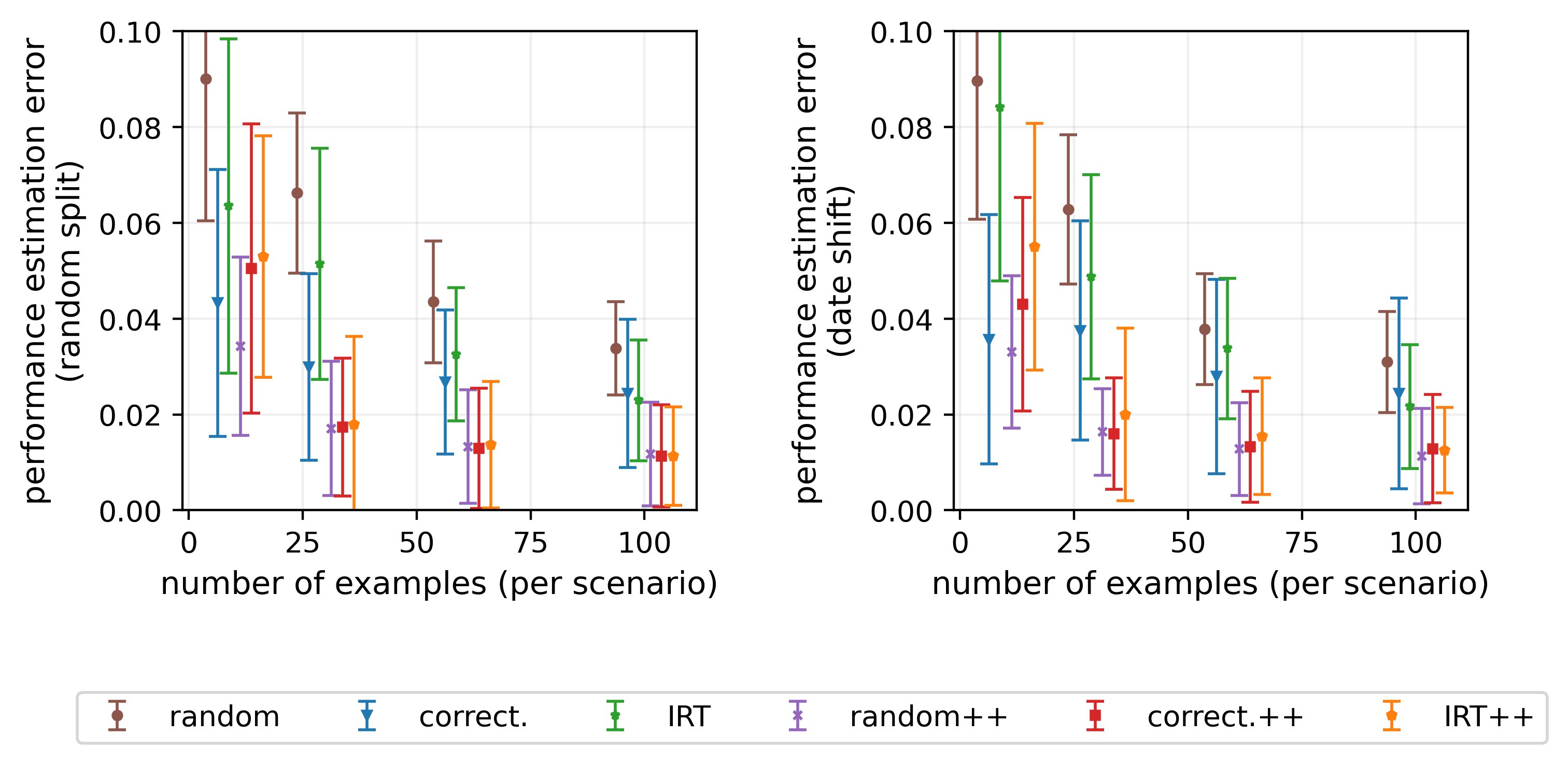}
\caption{Winogrande}
\end{figure}

\subsection{HELM}
\begin{figure}[H]
\centering
\includegraphics[width=.75\textwidth]{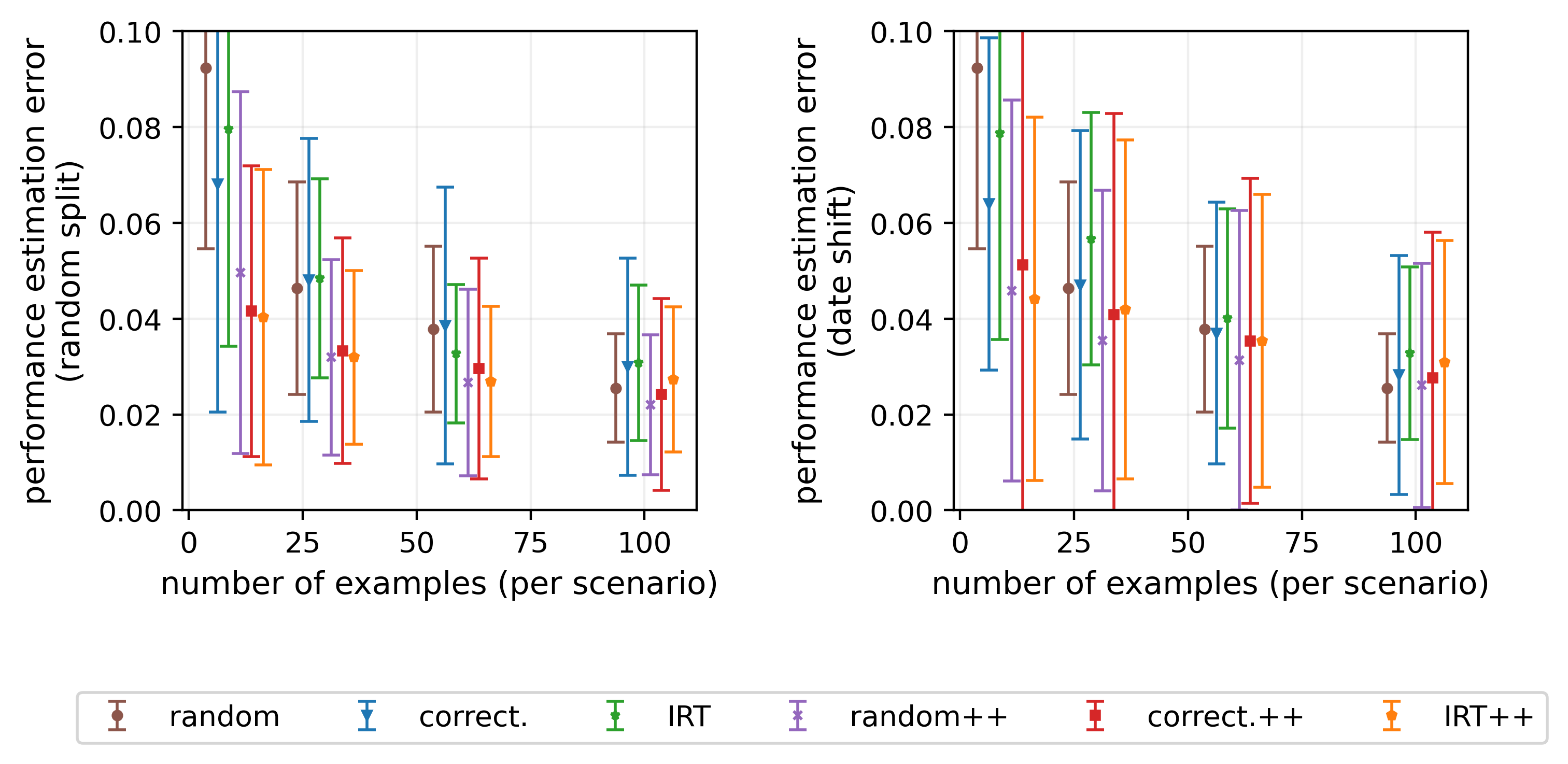}
\caption{OpenbookQA}
\end{figure}

\begin{figure}[H]
\centering
\includegraphics[width=.75\textwidth]{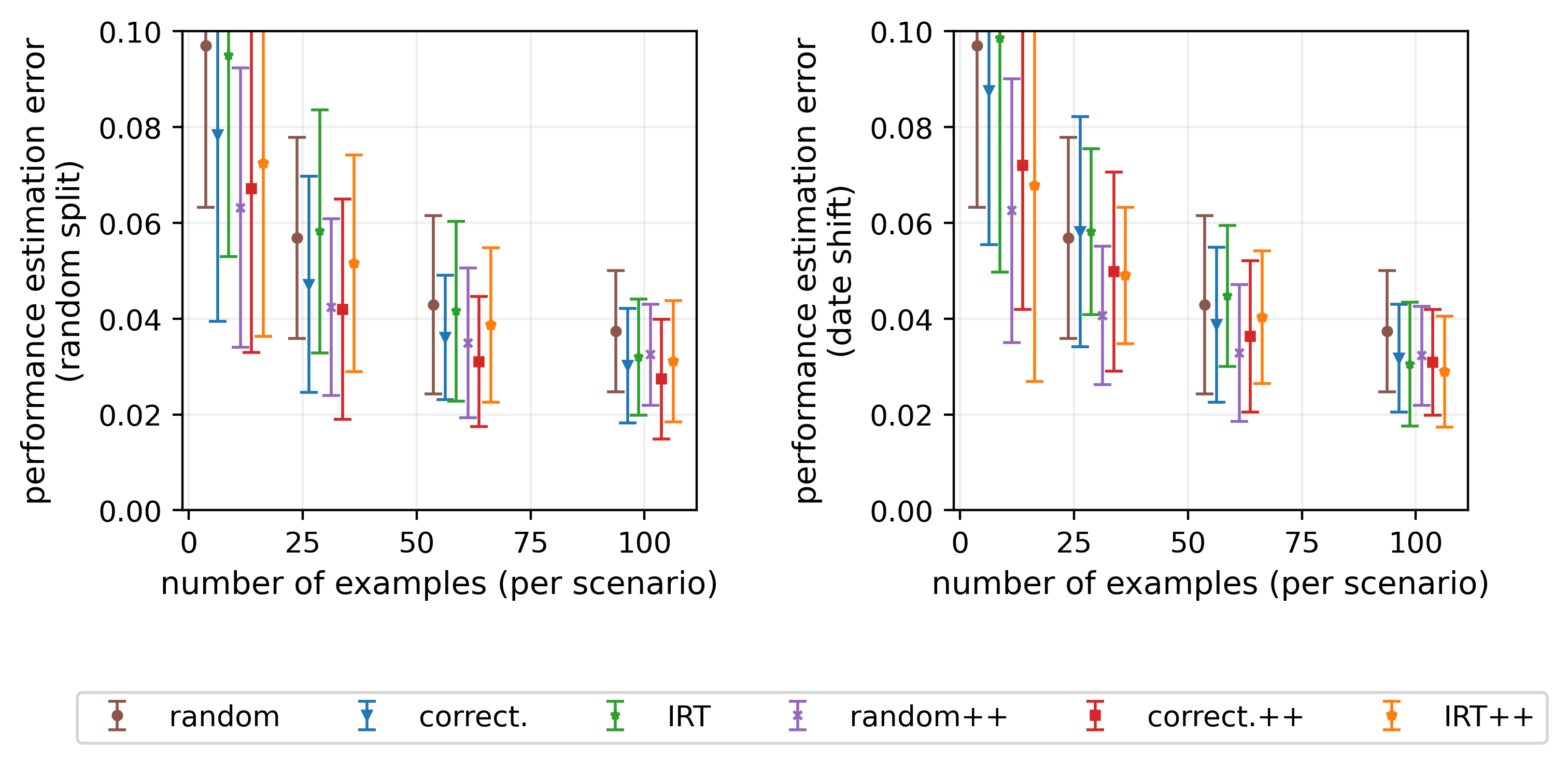}
\caption{GSM}
\end{figure}

\begin{figure}[H]
\centering
\includegraphics[width=.75\textwidth]{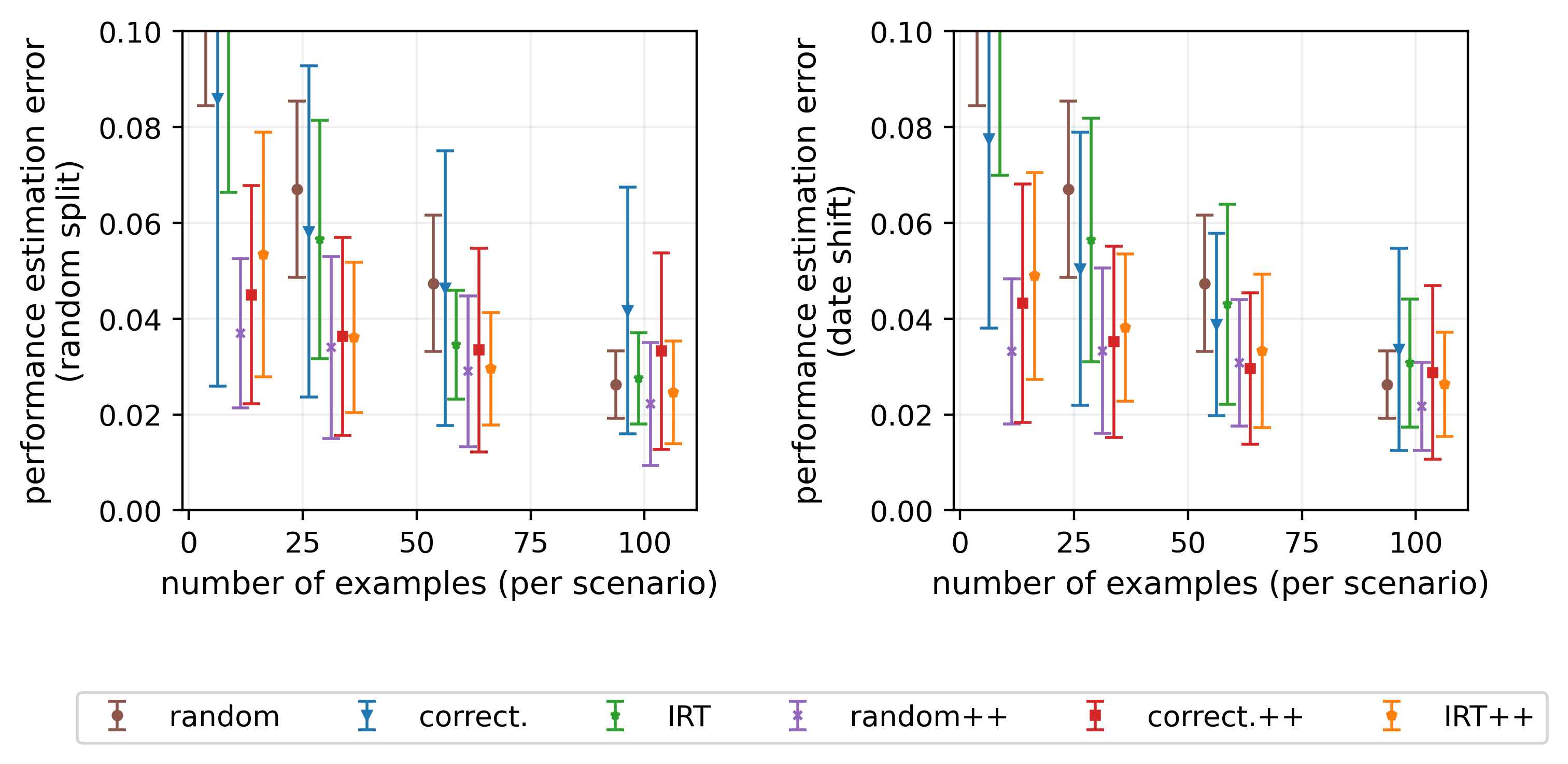}
\caption{LegalBench}
\end{figure}

\begin{figure}[H]
\centering
\includegraphics[width=.75\textwidth]{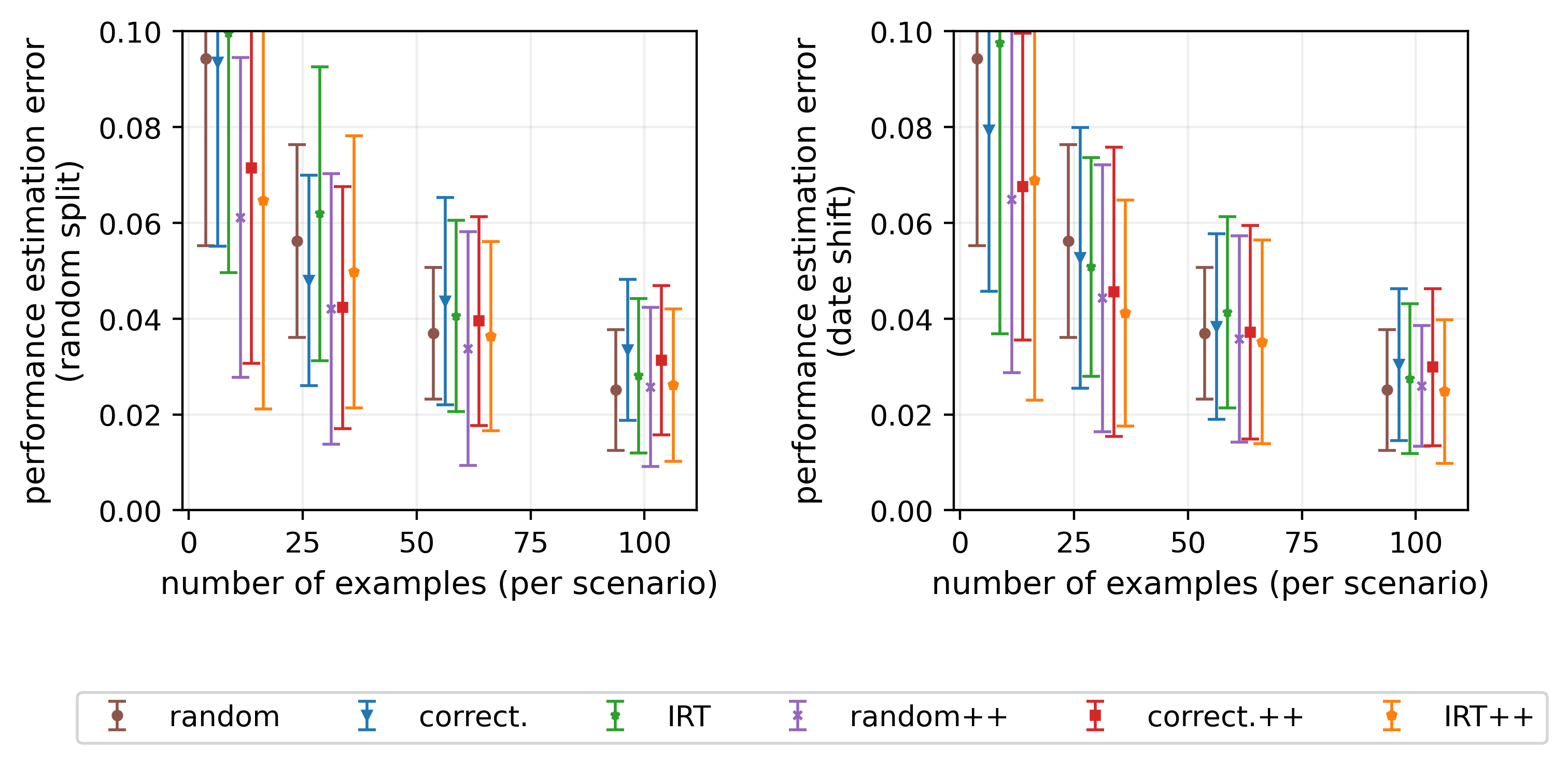}
\caption{Math}
\end{figure}

\begin{figure}[H]
\centering
\includegraphics[width=.75\textwidth]{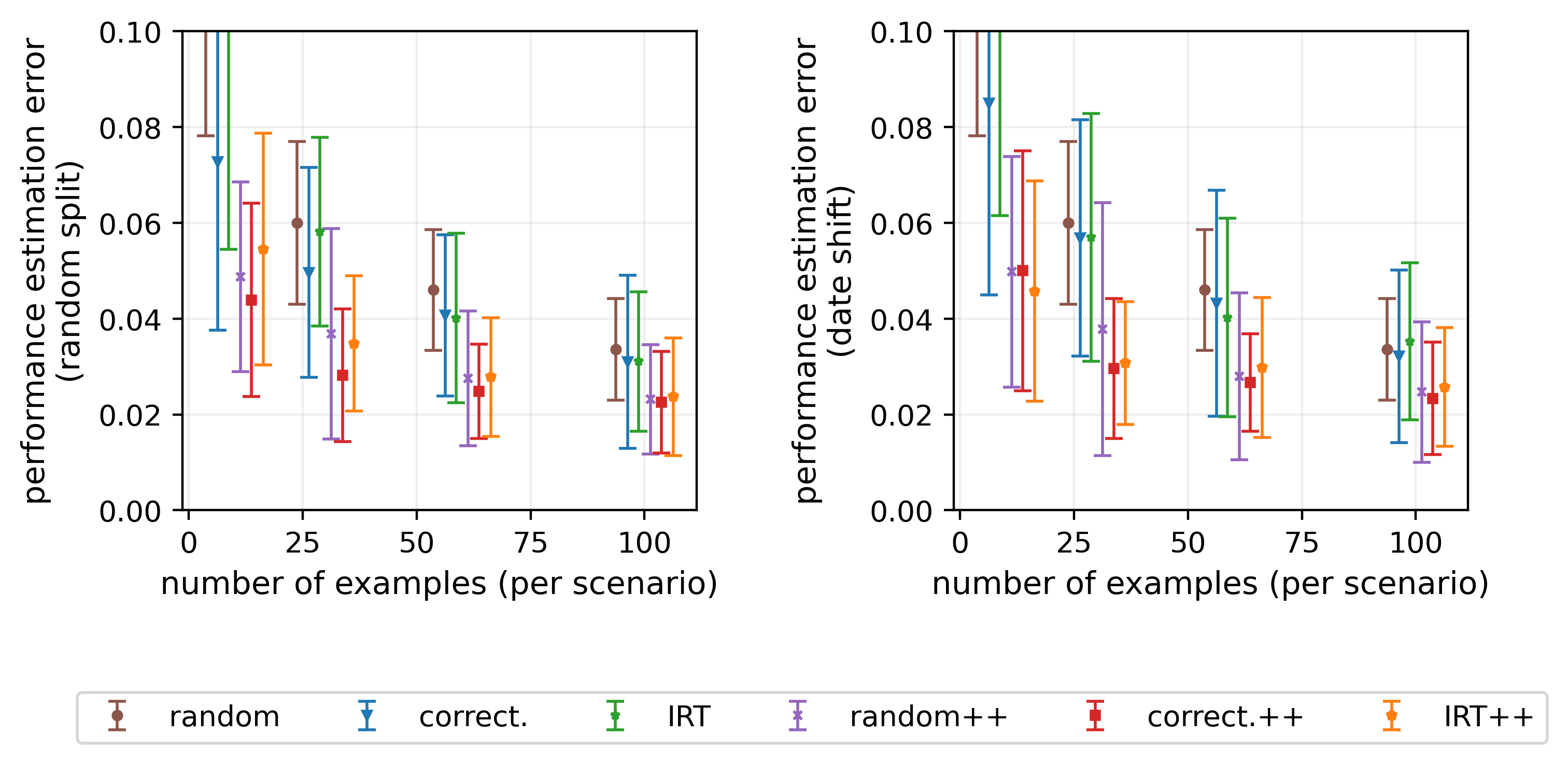}
\caption{MedQA}
\end{figure}

\begin{figure}[H]
\centering
\includegraphics[width=.75\textwidth]{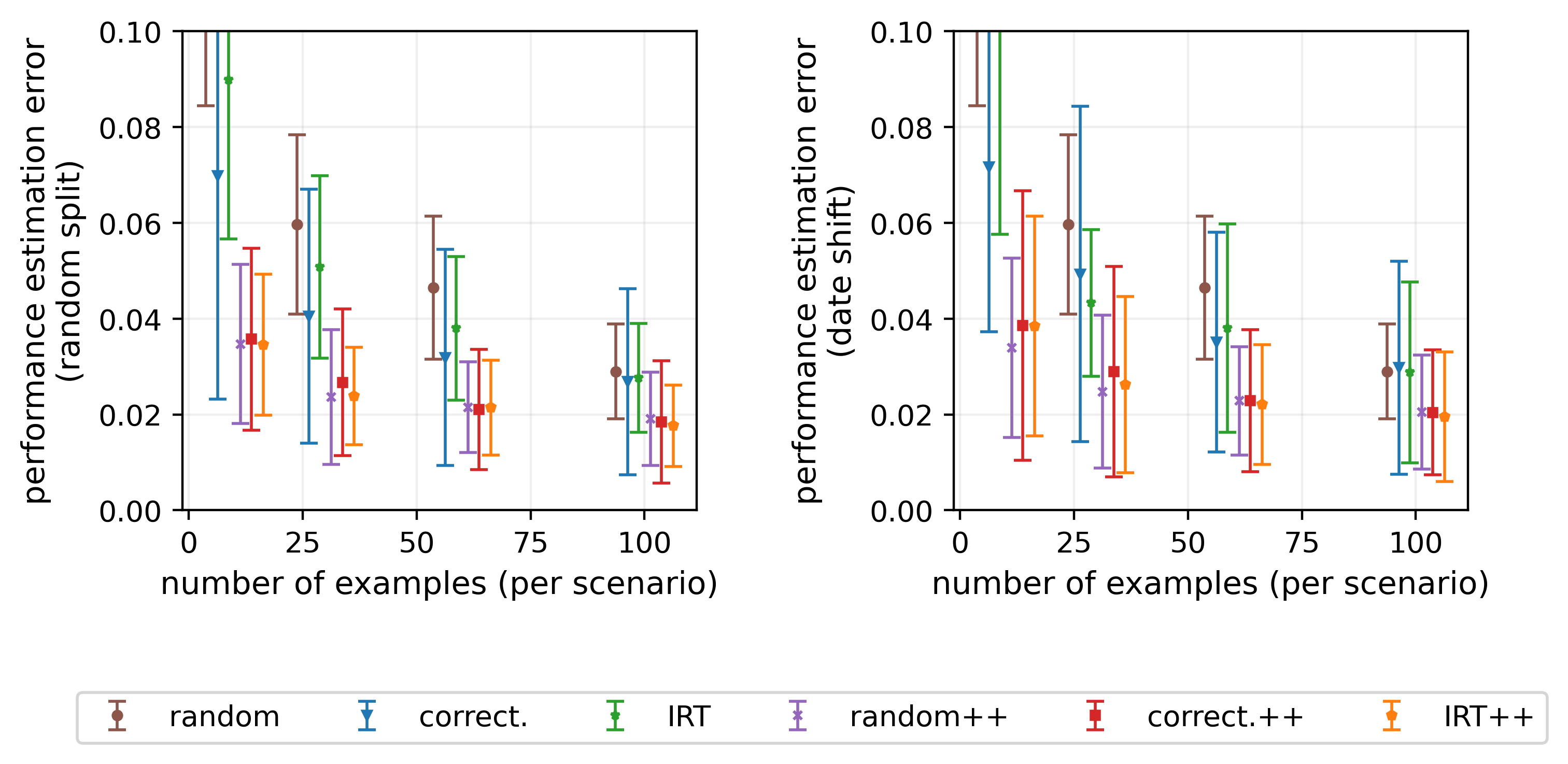}
\caption{MMLU}
\end{figure}

\begin{figure}[H]
\centering
\includegraphics[width=.75\textwidth]{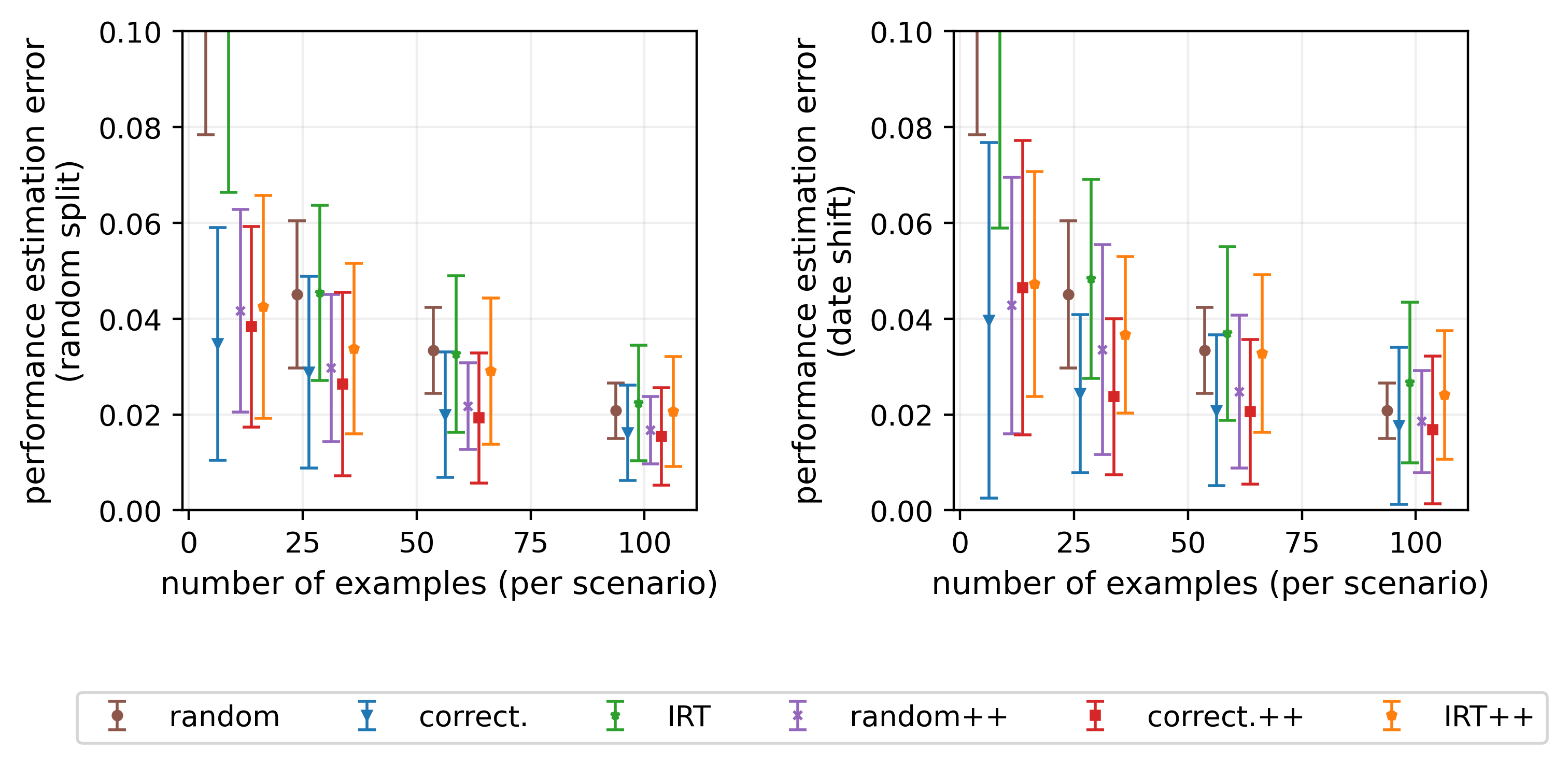}
\caption{NarrativeQA}
\end{figure}

\begin{figure}[H]
\centering
\includegraphics[width=.75\textwidth]{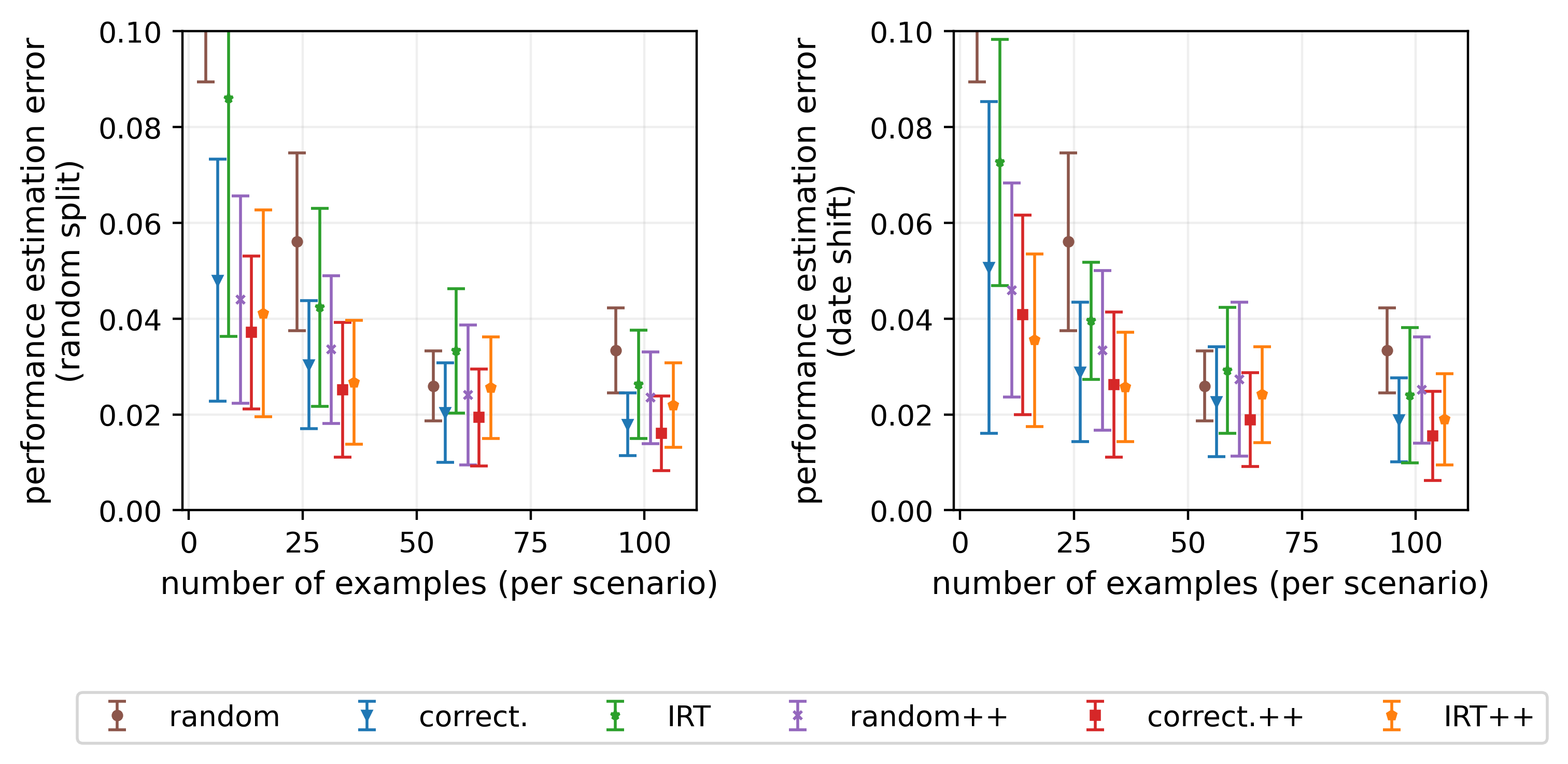}
\caption{NaturalQA (closed book)}
\end{figure}

\begin{figure}[H]
\centering
\includegraphics[width=.75\textwidth]{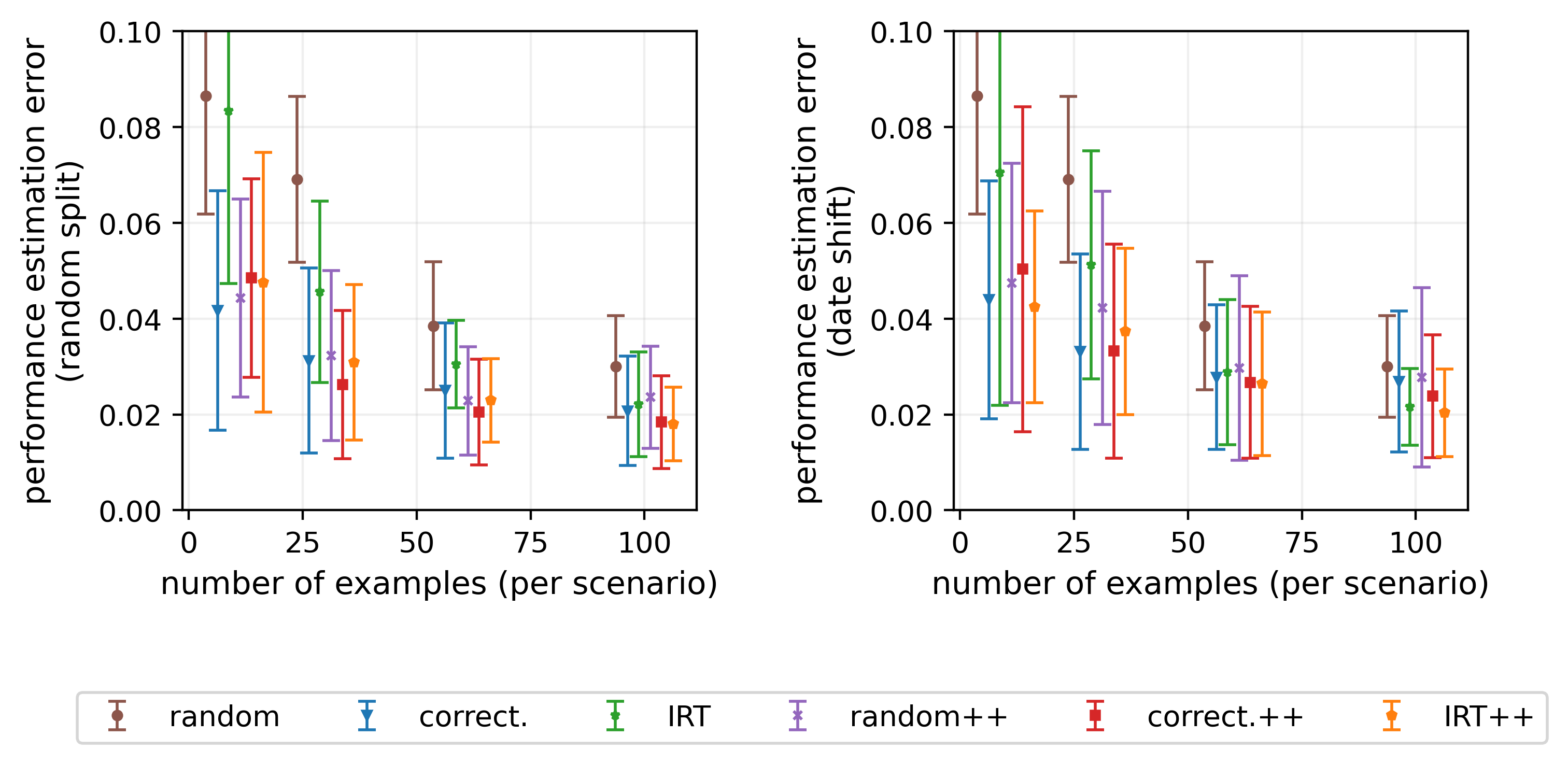}
\caption{NaturalQA (open book)}
\end{figure}

\begin{figure}[H]
\centering
\includegraphics[width=.75\textwidth]{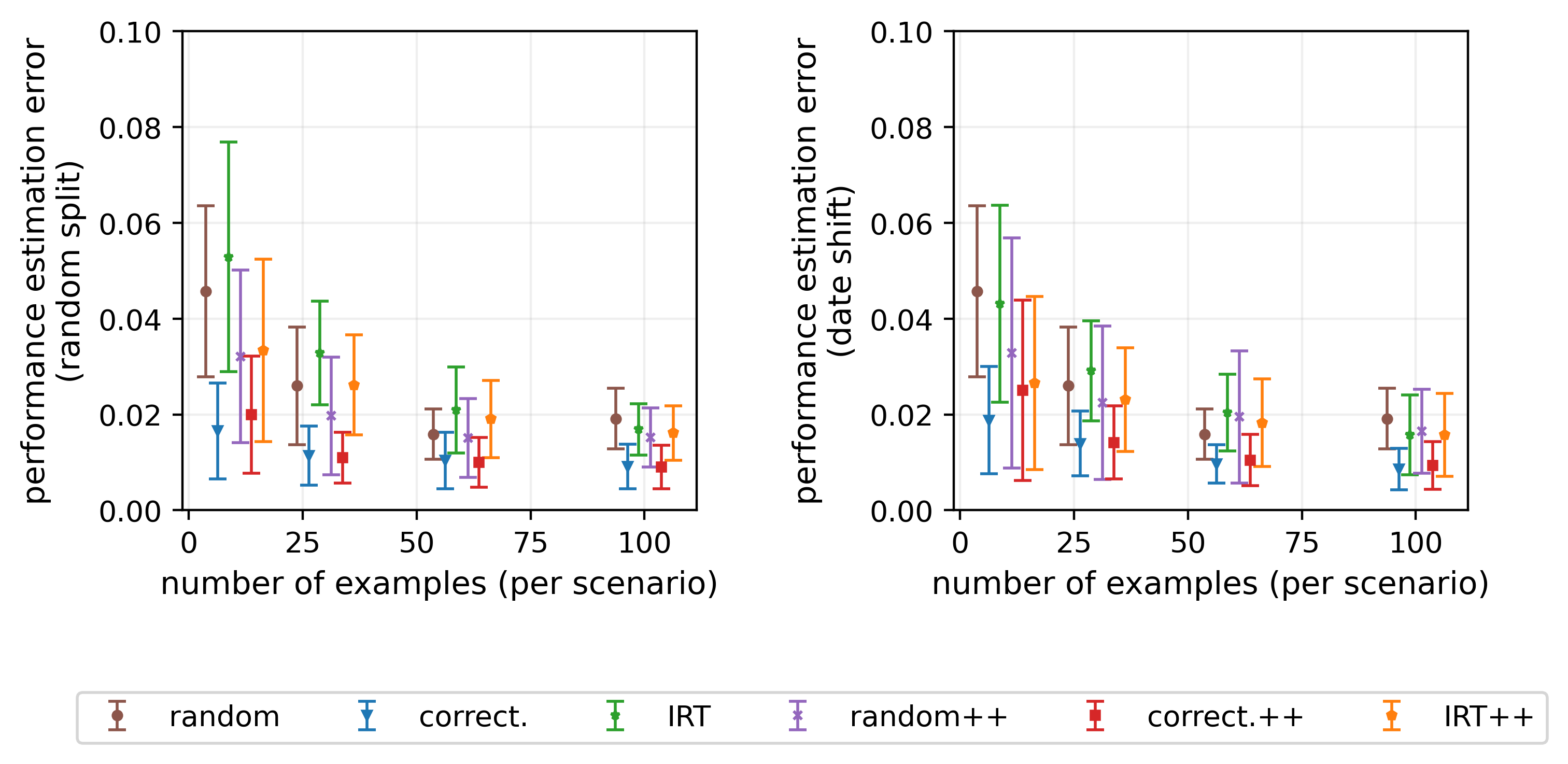}
\caption{WMT14}
\end{figure}

\end{document}